\documentclass[manuscript, screen]{acmart}
\AtBeginDocument{%
  }

\setcopyright{acmlicensed}
\copyrightyear{20}
\acmYear{2018}
\acmDOI{XXXXXXX.XXXXXXX}
\acmConference[Conference acronym 'XX]{Make sure to enter the correct
  conference title from your rights confirmation email}{June 03--05,
  2018}{Woodstock, NY}
\acmISBN{978-1-4503-XXXX-X/2018/06}

\author{Ioannis J Vourganas}
\authornote{Both authors contributed equally to this research.}
\email{ijvourganas@netrity.co.uk}
\orcid{0000-0002-3433-3757}
\affiliation{%
  \institution{Netrity Ltd}
  \city{Glasgow}
  \country{UK}
}

\author{Anna Lito Michala}
\authornote{Coresponding author.}
\authornotemark[1]
\affiliation{%
  \institution{University of Glasgow}
  \city{Glasgow}
  \country{UK}}
\email{annalito.michala@glasgow.ac.uk}
\orcid{0000-0001-7821-1279}

\renewcommand{\shortauthors}{Vourganas et al.}

\usepackage{amsthm}
\newtheorem{Theorem}{Theorem}
\usepackage{algorithm}
\usepackage{algorithmic}

\renewcommand\footnotetextcopyrightpermission[1]{} 
\begin{document}

\title[Stabilising Explainability Fragility in Cybersecurity AI]{Stabilising Explainability Fragility in Cybersecurity AI: The Impact and Mitigation of Multicollinearity in Public Benchmark Datasets}

\begin{abstract}
  This paper investigates a unexplored yet impactful vulnerability in AI explainability used in intrusion detection (IDS): multicollinearity-induced instability. Despite extensive reliance on post-hoc explainability tools such as SHAP or LIME, the impact of correlated features on explanation robustness is not evaluated. We introduce a formal theorem stating that multicollinearity inflates attribution variance. This demonstrates that explanations and feature importances are non-identifiable under multicollinearity. A suite of comprehensive experiments validates the theorem on a representative benchmark dataset, UNSW-NB15. Four widely used families of models are evaluated, including linear, tree-based, kernel, and neural, across full and pruned feature sets based on VIF and correlation thresholding. We propose the novel metric of Explanability Fragility Score and two novel methods to mitigate it with variable integration complexity. CAA-Filtering focuses on stabilising explanations by grouping attributions of trained models. SHARP is a novel training-time regularisation framework that penalises attribution instability, enabling controllable and monotonic improvement of explainability stability. The findings support stable predictive performance, using Kendall's $\tau$ to quantify instability across bootstrapped explanations. This work has direct implications for the trustworthiness and reproducibility of XAI in security-critical contexts, and motivates incorporating multicollinearity mitigations into the IDS pipelines, providing a set of guidelines for practitioners.
\end{abstract}



\received{6 December 2025}
\received[revised]{12 March 2009}
\received[accepted]{5 June 2009}

\maketitle

\section{Introduction}
\label{sec:intro}

Several datasets have become well-recognised benchmarks for intrusion detection research in recent years, including KDD’99, UNSW-NB15 \cite{moustafa2015unsw}, NSL-KDD, and DARPA. Since its introduction, the UNSW-NB15 dataset has become a cornerstone mainly due to its large number of attack labels \cite{vourganas2024applications}. Several researchers have authored a series of influential studies that established these datasets as relevant across IoT, SCADA, and industrial systems \cite{luqman2025intelligent, gunjal2023smart}. These works provided the community with modern traffic features and comprehensive benchmarking protocols, enabling hundreds of follow-up studies. However, as is typical in the dataset adoption, the focus was primarily on accuracy, detection rate, and feature engineering, with less emphasis on statistical dependencies among features.

In the last two decades, significant research has leveraged these well-known datasets across a wide spectrum of machine learning paradigms, from classical classifiers such as logistic regression, SVMs, and random forests \cite{meftah2019network} to deep learning approaches including CNNs and LSTMs \cite{thaljaoui2025intelligent}. Recently, XAI frameworks that report the importance of features using SHAP or LIME have also been used \cite{mane2021explaining}. Although these works demonstrate the versatility of the datasets and underscore their importance to the field, they share a common limitation: none report formal actions related to multicollinearity. Many refer to some form of feature-importance-related preprocessing. However, explicit checks such as variance inflation factors (VIF), feature drop or correlation thresholding, PCA for collinearity reduction or SHAP stability under collinearity are typically absent.

This omission leaves open the question of whether reported feature importance, explanation stability, and robustness claims are reliable. To our knowledge, no published work has provided a dataset-level multicollinearity audit of the selected dataset. Our work addresses this gap by introducing an axiomatic and computational framework that demonstrates how multicollinearity in any dataset induces fragility in explainability and robustness claims, thereby advancing the methodological rigour of IDS research.  

In this paper, we address this gap by formally analysing the impact of multicollinearity on the validity of AI-based intrusion detection research. We further establish this theoretical framework with experimental proof using UNSW-NB15 as the most representative dataset. Our contributions are:

\begin{itemize}
    \item Axiomatic framework. We introduce an axiomatic formulation of multicol-linearity-induced experimental fragility, establishing theoretical conditions under which feature attribution, model identifiability, and robustness claims become unreliable.
    \item Fragility metric. We propose a SHAP-based fragility score that captures the instability of explainable AI methods when applied to highly collinear feature spaces.
    \item Dataset-level audit. We present the first comprehensive multicollinearity audit of UNSW-NB15, quantifying feature redundancies through variance inflation factors (VIF) and correlation clustering.
    \item Empirical validation. We demonstrate the framework by training representative classifiers on UNSW-NB15, showing how feature importance, explanation stability, and predictive claims shift once collinearity is accounted for.
    \item Methodology. We introduce three methods to mitigate explainability fragility. The first requires VIF-based pruning of the input dataset. The second proposed the novel Colinearity-Aware-Attribution-Filtering applied on post-hoc filtering to group and stabilise explanation. The third introduces the novel SHAP Stability Regulariser (SHARP). This novel model-agnostic training-time objective directly penalises attribution fragility, enabling monotonic and controllable optimisation of explainability stability without requiring dataset pruning. We formally demonstrate $\lambda$-controlled explainability regularisation for SHARP, showing that attribution stability can be tuned independently from predictive performance.
    \item Results. We transform explainability stability from an evaluation metric into a controllable optimisation objective through SHARP, introducing attribution variance as a second-order regularisation constraint in parameter space. We provide the first formal $\lambda$-ablation study demonstrating monotonic explainability stabilisation under explicit regularisation control.
    \item Finally, we contribute guidelines for valid IDS benchmarking that can support explainability and EU AI Act requirements. 
\end{itemize}

Together, these contributions advance the methodological rigour of IDS benchmarking and provide a foundation for reproducible and trustworthy cybersecurity AI research. 




\section{Background and Related Work}

\subsection{Intrusion Detection Benchmarks and UNSW-NB15}
\label{sec:unsw}

Intrusion detection systems (IDS) are traditionally benchmarked using publicly available datasets. These datasets represent real-world or emulated network traffic under benign and adversarial conditions. Early benchmarks such as KDD’99 and its refined variant NSL-KDD offered a standardised evaluation ground but have been known to provide a skewed response distribution, redundancy, while missing modern attacks offering outdated traffic profiles \cite{divekar2018benchmarking}. To address these shortcomings, several datasets have been introduced in the recent decade, such as CICIDS2017, Bot-IoT, and UNSW-NB15 \cite{vourganas2024applications}.

Among them, UNSW-NB15 has been established as a cornerstone benchmark for IDS research, as is presented in the following paragraphs. Both the quantity and quality of the provided data have contributed to this establishment. The dataset combines 100 GB of network traffic covering a variety of scenarios from realistic benign flows to contemporary attack patterns labelled and categorised in nine families (e.g., fuzzers, backdoors, exploits, reconnaissance, shellcode, worms) \cite{moustafa2015unsw}. It includes 49 statistical and protocol-based features, covering byte counts, packet timings, connection states, and application-layer attributes. The dataset is partitioned into a training set (175,341 records) and a testing set (82,332 records).

In \cite{vibhute2024network}, the authors applied a two-stage method. Initially using a random forest for feature selection and subsequently a CNN model to detect anomalies using the UNSW-NB15 dataset and achieving 99\% accuracy. However, while random forests reduce dimensionality (selecting 15 out of 49 features), no analysis of multicollinearity is performed to assess the selected feature interdependences. This is a crucial omission, as correlated features can bias both feature importance and CNN performance. There is no mention of well used techniques for feature pre-processing, such as correlation heatmaps, VIF analysis, or decorrelation techniques such as PCA to ensure interpretability and model stability.

Similarly, in \cite{liu2025network}, a semi-supervised intrusion detection framework is presented, using contrastive learning and Bayesian Gaussian Mixture Models, train-ed on some of the most modern datasets (UNSW-NB15, NSL-KDD, CICIDS2017). While the method effectively reduces label dependency and handles imbalanced data, it does not address multicollinearity, despite using feature-dense datasets. In obscuring or overlooking pre-process, a model may learn redundant representations, impacting both efficiency and stability of the latent space. Moreover, reported improvements in accuracy do not garantee model stability in this respect. For example, when pair-wise correlation is addressed in pre-processing in \cite{more2024enhanced}, the accuracy is reduced compared to the values reported in \cite{liu2025network}. The authors acknowledge and assess correlation using Pearson Correlation Coefficient and Gain Ratio, and recommend dropping 7 of the highly correlated features, though others remain. However, they do not further explore multicollinearity, nor do they validate the impact of this filtering on model stability or embedding quality. 

Advanced pipelines of hybrid modes can further improve accuracy metrics, such as the one presented in \cite{zoghi2024building}. In this case, EDA is performed using Pearson correlation and Gain Ratio, removing several highly correlated features from the UNSW-NB15 dataset. While this improves preprocessing, the approach is static and dataset-specific, relying on one-time feature removal. Additionally, discarding correlated features risks losing variables that could still add value in non-linear settings. Compared to explainability-driven methods, this solution lacks dynamic feature assessment, cross-dataset generalizability, and post-hoc validation of interpretability stability.

In recent years attempts have been made to improve the explainability and interpretability of models. In \cite{keshk2023explainable}, the authors propose an explainable LSTM-based intrusion detection framework using SPIP (SHAP, PFI, ICE, PDP) to identify threats in IoT networks. They automate feature selection and highlight multicollinearity in the ToN\_IoT dataset. However, no analysis or addressing of multicollinearity in UNSW-NB15 is presented. This omission leads to unstable SHAP and PFI explanations and reduced overall explainability results for this dataset. Similarly, in \cite{ajagbe2024intrusion}, multiple ML classifiers are evaluated on UNSW-NB15, focusing on the challenges of class imbalance. While EDA is performed on class distribution, feature correlation and multicollinearity are not investigated.

Further in the direction of explainability research, the authors in \cite{hermosilla2025explainable} compare SHAP and LIME explanations for XGBoost and TabNet models trained on the UNSW-NB15 dataset, focusing on forensic interpretability. However, there is no exploratory data analysis or correlation checks. This risks unstable and biased explanations, as correlated features can distort SHAP and LIME outputs.

\subsection{Multicollinearity in Statistical Learning}
\label{sec:multico}

Multicollinearity is defined as a high correlation between two or more features in a dataset. Formally, considering a feature $x_i$ regressed against all other features in the set $X_{-i}$. The Variance Inflation Factor (VIF) is defined as:

\begin{equation}
    \text{VIF}(x_i) = \frac{1}{1 - R_i^2}
\end{equation}

where $R_i^2$ is the coefficient of determination from regressing $x_i$ on $X_{-i}$. Intuitively, $\text{VIF}(x_i)$ measures how much the variance of the estimated coefficient of $x_i$ is inflated due to linear dependence with other predictors. Diagnostic thresholds have been well established, regardless of some caution recommended in their use. The commonly used thresholds are $\text{VIF} > 5$ as moderate and $\text{VIF} > 10$ as severe multicollinearity \cite{james2021introduction, obrien2007caution}.

Multicollinearity has well-known consequences: unstable parameter estimates, inflated standard errors, and difficulty isolating the independent effect of a feature \cite{james2021introduction}. In particular, the effect of a feature is directly linked with explainability or interpretability for models. As a result, in machine learning, multicollinearity can result in feature importance variance, model identifiability issues, and invalidation of any generalisation claims \cite{hooker2021unrestricted}.  Particularly when explainable AI methods such as SHAP or LIME are considered, collinearity is linked to non-unique attribution scores, making them unstable and misleading \cite{kumar2020problems}.

Regardless of the well-known consequences reported in the statistics research community, the IDS research community \textbf{has rarely reported on multicollinearity of widely used cybersecurity datasets}. The majority of published IDS studies report some work on feature pre-processing or reduction, assuming that engineered traffic features are suitable for direct use in classifiers. To the best of the authors' knowledge, this is without reporting VIF or conducting collinearity reduction via PCA, Lasso, or correlation thresholding. \textbf{This paper aims to demonstrate that this omission has significant implications for the reproducibility and reliability of explainability claims in IDS research}.

\subsection{Summary and Motivation}
\label{sec:sum}

As it was estiblished in Section \ref{sec:unsw}, UNSW-NB15 has been used extensively in a wide range of AI intrusion detection approaches, from classical ML to deep learning and explainable AI (XAI). The majority of published work emphasises predictive performance in terms of accuracy, precision, recall, and F1-score, often exceeding 95\%. While valuable for demonstrating baseline efficacy of novel model architectures, these papers rarely investigate feature statistical dependencies. Higher-order models, such as deep or hybrid approaches, have expanded on this work to exploit nonlinear dependencies and improve robustness, again focusing on performance and often overlooking feature space reduction, implicitly assuming feature independence.

The dawn of XAI has seen XAI techniques interpreting UNSW-NB15-trained model predictions. Published outputs focus on feature rankings while ignoring the effects of multicollinearity on importance scores. Some studies have critiqued UNSW-NB15 and other benchmark datasets, highlighting issues with class imbalance, class overlap, and distribution drift \cite{ferrag2020deep}. However, even these papers leave multicollinearity largely unexamined. This systematic omission motivates our work. While prior work has acknowledged that correlated features may distort SHAP values, no existing IDS study provides a formal theorem linking VIF to attribution variance, nor introduces a trainable fragility regularisation objective. 


\section{Theoretical Framework}

This section formally defines multicollinearity and induced fragility in datasets available for AI training. 

\subsection{Problem Statement}

The hypothesis of this paper is that multicollinearity in the feature space of a cybersecurity dataset induces fragility in both model identifiability and explainability.
In other words, if a cybersecurity AI experiment utilises a dataset with high multicollinearity without further pre-processing, then its reported results (accuracy, feature importance, explainability) are likely flawed or misleading.

Formalising the original hypothesis, this section provides a theorem or statistical law in experimental AI validation.

\subsection{Definitions}

\textbf{Feature Space.}  
Let the dataset be represented as a feature matrix
\begin{equation}
    X \in \mathbb{R}^{n \times p}
\end{equation}

where $n$ is the number of observations and $p$ the number of features.

\textbf{Linear Dependence.}  
Multicollinearity implies the existence of (near) linear dependence among columns of $X$. Formally, there exist coefficients $\alpha_1, \ldots, \alpha_p$ not all zero such that
\begin{equation}
\sum_{j=1}^p \alpha_j X_j \approx 0
\end{equation}

where $X_j$ denotes the $j$-th feature vector.

\textbf{Variance Inflation Factor.}  
For each feature $X_j$, the Variance Inflation Factor (VIF) is defined as

\begin{equation}
\text{VIF}_j = \frac{1}{1 - R_j^2}
\end{equation}

where $R_j^2$ is the coefficient of determination obtained by regressing $X_j$ on all other features $\{X_k : k \neq j\}$.  
High $\text{VIF}_j$ indicates that $X_j$ is highly explained by other features.

\subsection{Theorem of Multicollinearity-Induced Experimental Fragility}

To further formalise the axiomatic definition, this section provides a theorem to express multicollinearity and instability based on feature attribution.

\begin{Theorem}[Multicollinearity-Induced Experimental Fragility]
\label{math:theorim}
Let $X \in \mathbb{R}^{n \times p}$ be the feature matrix, and $M$ a model trained on dataset $D = (X, y)$, where feature $x_i \in X$ has variance inflation factor $\text{VIF}(x_i)$. Then:

\begin{itemize}
  \item If $\text{VIF}(x_i) \to \infty$, the attribution of $x_i$ under any linear explanation method is non-identifiable. Thus, there exist multiple equivalent attribution vectors which can yield identical model outputs.
  
  \item Generalising, if $\text{VIF}(x_i) > \theta$, where $\theta \in [5, 10]$ is a diagnostic threshold, then the variance of attribution values for $x_i$ across bootstrap resamples satisfies:
  \begin{equation}
  \operatorname{Var}(\phi_i) \;\geq\; c \cdot (\text{VIF}(x_i) - 1)
  \end{equation}
  for some model-dependent constant $c > 0$.
\end{itemize}
\end{Theorem}

By definition, $\text{VIF}(x_i) = \tfrac{1}{1 - R_i^2}$, where $R_i^2$ is the coefficient of determination from regressing $x_i$ on $X_{-i}$. As $R_i^2 \to 1$, the design matrix becomes singular, yielding non-unique regression solutions and non-identifiable attribution vectors. For finite but large VIF, the variance of estimated coefficients (or equivalent SHAP values under linear models) is proportional to $\text{VIF}(x_i)$, establishing the lower bound. A detailed derivation is provided in Appendix~A. \qed 

This theorem establishes that collinearity directly destabilises explainability, linking a classical diagnostic (VIF) to modern attribution methods. It not only establishes the inevitability of attribution instability under multicollinearity but also motivates the introduction of fragility-aware optimisation objectives. In Section~\ref{sec:mitigation}, we operationalise this theoretical result via a novel method, converting fragility from a post-hoc observation into a trainable constraint.

\subsection{Explainability Fragility Score}
\label{sec:fragility}
Expanding further, it is possible, based on the theorem and related proof, to formally define a score for the attributed explainability instability. This can be achieved via attributions and is defined as:

\begin{equation}
    \text{Fragility}(x_i) = \frac{\operatorname{Var}(\phi_i)}{\mathbb{E}[|\phi_i|] + \epsilon}
\end{equation}

where $\phi_i$ denotes the attribution of feature $x_i$, and $\epsilon$ is a small constant added simply to ensure a result would always be possible even if $\mathbb{E}[|\phi_i|] = 0$. An attribution in the case of SHAP would be the SHAP value.

By definition, the Fragility Scores increases with $\operatorname{Var}(\phi_i)$, which in turn denotes unstable attributions. While the score is relative to the mean magnitude of attributions, which is linked to unreliable explanations. Consistent attributions would, as a result, generate low Fragility Score values.

Following the proof of Theorem~\ref{math:theorim}, a monotonic relationship is established between Fragility and VIF. Hence, a constant $k > 0$ exists such that:

\begin{equation}
    \text{Fragility}(x_j) \;\geq\; k \cdot (\text{VIF}(x_j) - 1)
\end{equation}

This relationship can be used to link classical statistical metrics used in dataset pre processing, such as VIF, with a robust and reliable metric for explanation instability. Thus, VIF quantifies structural dataset redundancy while the Fragility Score quantifies its manifestation in explainability. These two metrics can be used reproducibly as a unified auditing framework for modern ethical AI approaches, which meets the following criteria:

\begin{itemize}
  \item \textbf{Dataset audit:} identification of reliable features.
  \item \textbf{Explainability audit:} method agnostic Fragility Scores for reported attributions.
  \item \textbf{Trust:} interpretability should be reported only for features with bounded Fragility.
\end{itemize}

\subsection{Structural Equation Modelling Perspective}

The majority of AI/ML models share a common underlying assumption. This means that all features must be independent predictors. Exactly this assumption is often overlooked when datasets are provided as inputs. Theorem~\ref{math:theorim} stems from treating correlated features as independent, thus leading to instability. Structural Equation Modelling (SEM) complements the Theorem by providing a framework to transform features into independent predictors via the introduction of latent constructs.

As proven by Theorem~\ref{math:theorim}, two nearly collinear features $X_1$ and $X_2$ with  ($\operatorname{Cov}(X_1, X_2) \approx 1$) lead to $\operatorname{Var}(\hat{\beta}_j)$ diverging as $\text{VIF}_j \to \infty$. Hence, any causal or explanatory claims about their individual contributions are invalid. In SEM the features $X_1$ and $X_2$ would be removed from the dataset as their correlation indicates an underlying latent feature $X_3$. Then $X_3$ would be introduced to the dataset instead. Thus, a reduction of variance inflation is observed. This new independent latent predictor does not introduce identifiable coefficients which are estimated on orthogonal latent variables \cite{hair2021introduction}.

Hence, the SEM perspective aligns perfectly with the proposed Explainability Fragility Score. In both cases high-VIF groups can be derived. In the case or SEM, latent constructs can replace those groups thus restoring stability in regression coefficients, leading to reliable SHAP attributions. Such a synergy between the proposed Explainability Frangility Score and SEM approach in any modern AI pipeline can ensure interpretability and compliance. In this work we will provide a glimpse into these high-VIF clusters. However, this research avenue is beyond the scope of this work, and would be an excellent future direction.

\section{Dataset and Feature Analysis}

\subsection{Dataset Overview}

As established in earlier sections, the UNSW Canberra Cyber Range Lab UNSW-NB15 is a widely used benchmark for IDS research. The data set was generated using the IXIA PerfectStorm tool and covers realistic benign traffic and synthetic traffic of contemporary attack scenarios. The raw traffic recorded is in pcap format and is approximately 100~GB. The raw traffic data were further processed using Argus and Bro-IDS to extract  2,540,044 records \cite{moustafa2015unsw}.

The dataset provides 49 statistical and protocol-derived features. The labeled records can be categoried in nine attack families as well as normal traffic with a binary label ($y \in \{0, 1\}$). A further label differentiates between the attack families of Fuzzers, Analysis, Backdoors, DoS, Exploits, Generic, Reconnaissance, Shellcode, and Worms (\texttt{attack\_cat}). Additionally, to aid reproducibility, the dataset is often provided in a reduced scope pre-partitioned with a training set of 175,341 records and a testing set of 82,332 records. 

\subsection{Feature Groups}
This broad coverage of the features makes UNSW-NB15 a rich source for IDS benchmarking. The 49 features can be categorised as follows:

\begin{itemize}
  \item \textbf{Flow-based features:} byte and packet counts (\texttt{sbytes}, \texttt{dbytes}, \texttt{spkts}, \texttt{dpkts}), rates (\texttt{rate}, \texttt{sload}, \texttt{dload}).
  
  \item \textbf{Time-based features:} inter-arrival times and jitter statistics (\texttt{dur}, \texttt{sinpkt}, \texttt{dinpkt}, \texttt{sjit}, \texttt{djit}).
  
  \item \textbf{TCP/connection state features:} flags and acknowledgments (\texttt{sttl}, \texttt{dttl}, \texttt{tcprtt}, \texttt{synack}, \texttt{ackdat}).
  
  \item \textbf{Content-based features:} application-layer interactions (\texttt{is\_ftp\_login}, \texttt{ct\_ftp\_cmd}, \texttt{response\_body\_len}).
  
  \item \textbf{Contextual features:} host/service interactions (\texttt{ct\_srv\_src}, \texttt{ct\_state\_ttl}, \texttt{ct\_dst\_ltm}, \texttt{ct\_srv\_dst}, \\ \texttt{is\_sm\_ips\_ports}).
\end{itemize}

However, this versatility and broadness can also increase the likelihood of feature redundancy: for example, \texttt{sbytes} and \texttt{dbytes} can be reasonably related to \texttt{spkts} and \texttt{dpkts} as there is an analogy between packet size and bytes. Similarly, \texttt{tcprtt}, \texttt{synack}, and \texttt{ackdat} are functionally related. Previous studies have taken advantage of the feature richness of UNSW-NB15, without systematically addressing these redundancies. Where attempts are made, a clear methodology is not followed consistently. 

We therefore proceed to quantify collinearity in the data set using correlation clustering and variance inflation factors, and to assess its impact on model explainability through our proposed \textit{ Explanability Fragility Score}.

\subsection{Multicollinearity Audit}

Initially, we approach UNSW-NB15 classically to quantify correlation and multicollinearity. As starndard practice indicates, the identifiers and categorical fields such as \texttt{proto}, \texttt{service}, and \texttt{state} are excluded. Thus, the analysis is restricted to the 39 numerical features.

\subsubsection{Correlation Analysis}

Figure~\ref{fig:pearson} presents the Pearson correlation heatmap for the numeric feature space. 

\begin{figure}[h]
  \centering
  \includegraphics[width=0.9\linewidth]{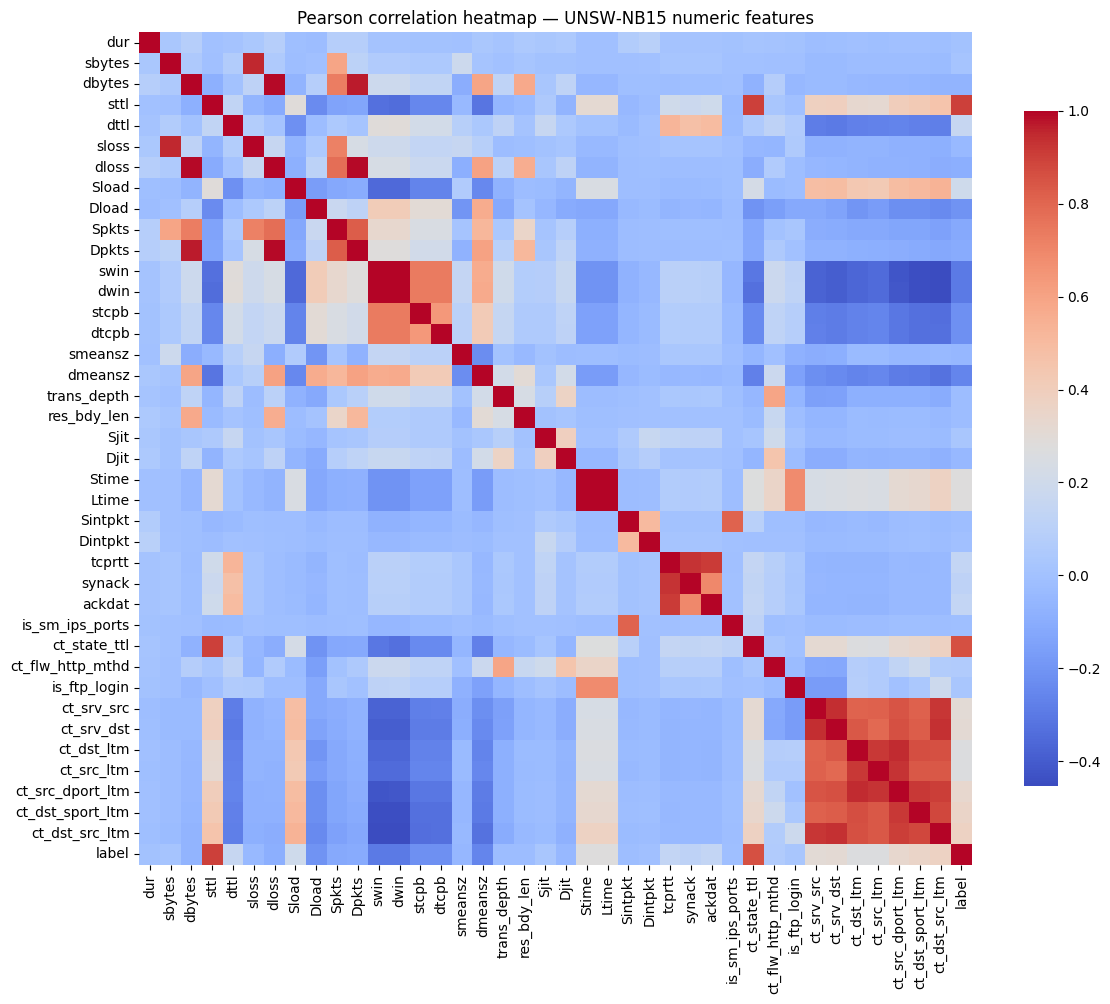}
  \caption{Pearson correlation heatmap for the numeric feature space}
  \Description{Pearson correlation heatmap for the numeric feature space, visually identifying the clusters that are described in the section Dataset Overview.}
  \label{fig:pearson}
\end{figure}

As expected from the logical feature groups, the correlation analysis also reveals distinct correlation clusters:

\begin{itemize}
  \item \textbf{Flow-based cluster:} \texttt{sbytes}, \texttt{dbytes}, \texttt{spkts}, \texttt{dpkts}, \texttt{sload}, \texttt{dload}, and \texttt{rate} have pairwise correlations exceeding 0.85. This is in line with the expected logical redundancy of byte and packet count features.
  
  \item \textbf{Timing cluster:} \texttt{tcprtt}, \texttt{synack}, \texttt{ackdat}, \texttt{sinpkt}, and \texttt{dinpkt} clearly demonstrate the functional relationships in round-trip and inter-arrival timing.
  
  \item \textbf{Contextual cluster:} features starting with \texttt{ct\_} such as \texttt{ct\_srv\_src}, \texttt{ct\_srv\_dst}, and \texttt{ct\_state\_ttl} are functionally correlated in their tracking of host-level and service-level state transitions.
\end{itemize}

These 15 out of 39 features (or 38\% of the feature space), form 3 clusters which will negatively impact downstream modelling and explainability.

\subsubsection{Variance Inflation Factor}
\label{sec:VIF}
The VIF is computed to further investigate multicolliniearity on a representative sample of 5,000 records for computational tractability. Table~\ref{tab:topvif} reports the to 10 features by VIF in decending order along with the observed interpretation. It is critical to highlight that some of these features were \textbf{not previously identified in the correlation analysis} phase and do not belong to any of the derived clusters. 

\begin{table}[H]
\centering
\caption{Top 10 features by Variance Inflation Factor.}
\label{tab:topvif}
\begin{tabular}{lcp{7cm}}
\toprule
\textbf{Feature} & \textbf{VIF} & \textbf{Interpretation} \\
\midrule
\texttt{is\_ftp\_login} & $\infty$ & Binary indicator perfectly determined by other FTP-related features \\
\texttt{ct\_ftp\_cmd}   & $\infty$ & Redundant with \texttt{is\_ftp\_login} \\
\texttt{tcprtt}         & $\infty$ & Perfectly predictable from \texttt{synack} and \texttt{ackdat} \\
\texttt{ackdat}         & $\sim 9 \times 10^{15}$ & Functionally redundant with \texttt{synack} \\
\texttt{synack}         & $\sim 9 \times 10^{15}$ & Functionally redundant with \texttt{ackdat} \\
\texttt{dloss}          & $5.0 \times 10^3$ & Highly correlated with flow packet counts \\
\texttt{dpkts}          & $3.5 \times 10^3$ & Redundant with \texttt{spkts} and \texttt{sbytes} \\
\texttt{dbytes}         & $8.4 \times 10^2$ & Redundant with \texttt{sbytes}, \texttt{spkts} \\
\texttt{sloss}          & $4.9 \times 10^2$ & Flow loss highly tied to packet/byte counts \\
\texttt{sbytes}         & $4.5 \times 10^2$ & Strongly correlated with \texttt{dbytes}, \texttt{spkts} \\
\bottomrule
\end{tabular}
\end{table}

These top 10 features exhibit extreme multicollinearity. The VIF values are orders of magnitude beyond standard or expected thresholds. Some are expected results based on the logical grouping and the correlation analysis.
For exaple, the TCP timing triplet of \texttt{tcprtt}, \texttt{synack}, \texttt{ackdat} was earlier identified during both grouping analyses. 
This triplet is a perfectly collinear group. 
Another common suspect would be the byte/packet counts, which forms another collinear group. 
However, a perfectly collinear group that was not previously picked up during correlation analysis was the FTP-related features
\texttt{is\_ftp\_login} and \texttt{ct\_ftp\_cmd}.

The results of this analysis can become inputs for SEM to replace such clusters with meaningful latent features. For example, the features \texttt{sbytes}, \texttt{dbytes}, \texttt{spkts}, and \texttt{dpkts} can be understood as multiple measurements of a latent construct ``flow volume''. Similarly, the triplet \texttt{tcprtt}, \texttt{synack}, and \texttt{ackdat} can be viewed as indicators of a latent ``round-trip time'' construct. 

The output of such an SEM approach however, would again be subject to VIF computation to verify that the latent feature is not now collinear with others. We propose this as an iterative process until the dataset is truly formed of independent predictors.

The results of the multicollinearity audit demonstrate the inherent redundancy of the UNSW-NB15. Even though some modelling approaches, such as tree-based ensembles and deep learning modes are able to tolerate correlated inputs this outcome significantly undermines identifiability, explainability and reproducibility. For the research community, this translates in unstable regression coefficients, unreliable SHAP and LIME attributions, and unreliable feature importance reported in prior IDS papers without full collinearity controls.

These results provide a validation of Theorem~\ref{math:theorim} and further motivation for the use of the proposed \textit{Explainability Fragility Score} to quantify this instability in explainable AI applied to IDS.

\section{Empirical Explainability Fragility Analysis}

To provide a metric that reliably describes the extent of instability in explainable AI in IDS and to formally assess the consequences of multicollinearity, this work proposes evaluating feature attribution stability using the Explainability Fragility Score presented in Section~\ref{sec:fragility}. This section presents the analysis of empirical experiment results over a comprehensive list of representative AI/ML models prevalent in modern IDS research. The unprocessed results are presented in Appendix B. 

\subsection{Experimental Setup}

The aim of the study is to evaluate both performance and explainability as impacted by existing multicollinearity in benchmark datasets. Hence, the study provides a comparative analysis of a control and a hypothesis scenario using UNSW-NB15:

\begin{itemize}
  \item \textbf{Control: Full feature set} which represents one of the standard approaches in prior IDS literature.
  \item \textbf{Hypothesis: Reduced feature set} whereby multicollinearity is addressed with the removal of high-VIF features ($\text{VIF} > 10$) and high-correlation features ($|\rho| > 0.85$). VIF values were as reported in the dataset audit in Section \ref{sec:VIF}.
\end{itemize}

A set of representative classifiers is selected to cover as wide a spectrum of AI/ML IDS approaches as possible. This includes Logistic Regression (LR), XGBoost, Random Forest,  Multi-Layer Perceptron Neural Network (MLP NN), Convolutional Neural Network (CNN), Long Short-Term Memory network (LSTM) as a representative of Recurrent NNs, and a Radial Basis Function (RBF) kernel Support Vector Machine (SVM) suited for non-linearly-separable data. The LR approach is the most illustrative case for coefficient instability; however, all other cases are potentially affected. XGBoost represents an ensemble-based baseline often utilised in IDS research due to the performance being unaffected by collinearity; however, it can potentially illustrate the effect of feature attributions. The selected Neural Network approaches are among the most prevalent AI methods with robust performance reported in modern studies. 

To enable such a wide comparison and to ensure reproducibility of the results, 10 runs were executed for both training and testing. SHAP values were computed across multiple bootstrap resamples, where 10 runs were performed on 10,000 records each. Fragility Scores were calculated for each feature in both the Control and the Hypothesis set each time. For each model, a ranked list of Fragility Scores was generated at the end of the experiment, demonstrating the most unstable attributions.

The following metrics were used for evaluation: 
\begin{itemize}
    \item \textbf{Predictive performance}: Mean values for  Accuracy, Precision, Recall, F1 score, ROC AUC;
    \item \textbf{Explainability}: Fragility Scores for each feature, Kendal’s $\tau$  calculated for comparison of feature importance stability across the top 20 and top 50 features.
\end{itemize}

The metrics for each model were recorded before and after multicollinearity reduction using VIF-based feature pruning, as were the fragility scores. The most fragile features for each model are reported in Appendix B before and after dataset treatment.   

\subsection{Predictive Performance}

Table~\ref{tab:perform} reports baseline (control) and comparative (hypothesis) mean predictive metrics on the test set. The control and hypothesis results were compared with a paired t-test with 9 degrees of freedom. The resulting p-values for Accuracy, Precision, Recall, F1-Score and ROC AUC never exceeded the threshold for statistical significance ($p-value < 0.05$). The percentage drop in mean performance is presented in Table~\ref{tab:perdrop}.

\begin{table}[H]
\centering
\caption{Predictive performance metrics under collinearity-controlled ($M_h$) and full ($M_c$) feature sets; where $M$ is the model's name. Values are reported as percentages for all metrics except ROC AUC, which is reported as a proportion.}
\label{tab:perform}
\begin{tabular}{lccccc}
\toprule
\textbf{Model} & \textbf{Accuracy} & \textbf{Precision} & \textbf{Recall} & \textbf{F1-score} & \textbf{ROC AUC} \\
\midrule
LR$_c$ & 80.7\% & 78.8\% & 95.6\% & 86.4\% & 0.750 \\
LR$_h$ & 78.9\% & 77.8\% & 93.8\% & 85.1\% & 0.732 \\
\hline
XGBoost$_c$ & 94.4\% & 96.1\% & 95.0\% & 95.6\% & 0.94 \\
XGBoost$_h$ & 93.0\% & 94.4\% & 94.6\% & 94.5\% & 0.92 \\
\hline
Rand. Forest$_c$ & 94.7\% & 96.0\% & 95.6\% & 95.8\% & 0.943 \\
Rand. Forest$_h$ & 92.5\% & 94.5\% & 93.8\% & 94.1\% & 0.920 \\
\hline
MLP NN$_c$ & 92.8\% & 94.3\% & 94.4\% & 94.4\% & 0.922 \\
MLP NN$_h$ & 88.8\% & 90.3\% & 92.4\% & 91.3\% & 0.874 \\
\hline
CNN$_c$ & 92.9\% & 94.2\% & 94.9\% & 94.5\% & 0.922 \\
CNN$_h$ & 86.2\%  & 87.9\% & 90.8\% & 89.3\% & 0.843 \\
\hline
LSTMs$_c$ & 94.9\% & 96.9\% & 95.0\% & 95.9\% & 0.991 \\
LSTMs$_h$ & 94.7\% & 96.1\% & 95.6\% & 95.9\% & 0.991 \\
\hline
RBF SVM$_c$ & 89.7\% & 86.6\% & 99.2\% & 92.5\% & 0.972 \\
RBF SVM$_h$ & 89.7\% & 86.6\% & 99.2\% & 92.5\% & 0.972 \\
\bottomrule
\end{tabular}
\end{table}

\begin{table}[H]
\centering
\caption{Percentage drop in performance metrics between collinearity-controlled and full feature sets (negative numbers indicate rise instead).}
\begin{tabular}{lccccc}
\toprule
\textbf{Model} & \textbf{Accuracy} & \textbf{Precision} & \textbf{Recall} & \textbf{F1-score} & \textbf{ROC AUC} \\
\midrule
LR & 1.8\% & 1.0\% & 1.8\% & 1.3\% & 1.8\% \\
XGBoost & 1.4\% & 1.7\% & 0.4\% & 1.1\% & 2.0\% \\
Rand. Forest & 2.2\% & 1.5\% & 1.8\% & 1.7\% & 2.3\% \\
MLP NN & 4.0\% & 4.0\% & 2.0\% & 3.1\% & 4.8\% \\
CNN & 6.8\% & 6.3\% & 4.0\% & 5.2\% & 7.9\% \\
LSTMs & 0.02\% & 0.08\% & -0.6\% & 0.0\% & 0.0\% \\
RBF SVM & 0.0\% & 0.0\% & 0.0\% & 0.0\% & 0.0\% \\
\bottomrule
\end{tabular}
\label{tab:perdrop}
\end{table}

Thus, predictive performance remains largely stable even after feature reduction. This demonstrates that high-VIF features are redundant even though there is a slight reduction across all metrics for all models. CNNs have a much higher percentage drop, but this is largely due to the larger standard deviation in performance metrics between experiment runs.  

\subsection{Explainability Stability}

The impact on explainability is much more profound, however. Table~\ref{tab:fscore} reports the most often encountered features by Fragility Score across all models. Detailed Fragility Scores are reported in Appendix~B for each of the models, respectively.

\begin{table}[H]
\centering
\caption{The 5 features most often encountered in the top-10 featured rankings by Fragility Score for all models.}
\label{tab:fscore}
\begin{tabular}{lcp{7cm}}
\toprule
\textbf{Feature} & \textbf{Fragility Score} & \textbf{Interpretation} \\
\midrule
\texttt{sttl} & $\gg 1$     & Collinear with \texttt{tcprtt} \\
\texttt{synack} & $\gg 1$     & Swaps attribution weight with \texttt{ackdat} \\
\texttt{sbytes} & $\gg 1$     & Redundant with \texttt{dbytes}, \texttt{spkts} \\
\texttt{dbytes} & $\gg 1$     & Redundant with \texttt{sbytes} \\
\texttt{dload}  & $> 1$       & Derived from packet counts \\
\texttt{sloss}  & $> 1$       & Derived from packet counts \\
\bottomrule
\end{tabular}
\end{table}

From the analysis of the Explainability Fragility Scores presented in Appendix B, several observations can be made: 

\begin{itemize}
  \item Perfectly collinear triplets (e.g., \texttt{tcprtt}, \texttt{synack}, \texttt{ackdat}) yield highly unstable SHAP values, with attribution oscillating arbitrarily among them across runs. At least one of these features is present in the rankings of high Fragility Scores for each of the models. 
  
  \item Redundant flow features (bytes and packets) share similar fragility, as models can substitute any of them without altering predictions. One manifestation of these redundant features is always present in the top fragile feature rankings with a high fragility score. 
  
  \item A high Fragility Score is not directly related to the feature with the highest VIF. Thus the Fragility Score is a manifestation of the symptom, not the root cause. 

\end{itemize}

It is further observed that Fragility Scores reduce significantly for the Hypothesis dataset. In the control set, Fragility Scores can be orders of magnitude larger than the Hypothesis set, especially for some of the highly collinear features. While features that are relatively stable often have fragility scores lower than $1$. Moreover, some models are better suited for collinearity, and this is demonstrated with low fragility scores already in the control set. However, even for those cases, Fragility Scores significantly reduce in the Hypothesis set.

To formally assess the impact of VIF-based feature pruning, Table~\ref{tab:kendal} presents the calculated Kendall's $\tau$ rank correlation for SHAP feature importance stability in all tested models for both the control and the hypothesis data sets. Higher values indicate more stable feature rankings, while lower values indicate low stability and thus fragile explanations. Kendall's $\tau$ values closer to $0$ demonstrate that features (such as \texttt{sbytes}, \texttt{dbytes}, \texttt{spkts}, \texttt{tcprtt}, \texttt{ackdat}) fluctuate in position, with attribution mass oscillating. Although Kendall's $\tau$ values closer to $1$ indicate that rankings are stabilised, and the features remaining after pruning demonstrate stable SHAP values.

\begin{table}[H]
\centering
\caption{Kendall's $\tau$ calculated across the top-50 important features for each of the tested models for both the full (control $_c$) and VIF-based pruned (hypothesis $_h$) datasets. Higher values imply higher stability.}
\label{tab:kendal}
\begin{tabular}{lcc}
\toprule
\textbf{Model} & \textbf{$\tau_c$} & \textbf{$\tau_h$} \\
\midrule
LR & 0.37 & 0.91 \\
XGBoost & 0.44 & 0.68 \\
Rand. Forest & 0.11 & 0.51  \\
MLP NN & 0.63 & 0.73 \\
CNN & 0.04 & 0.34 \\
LSTMs & 0.14 & 0.17 \\
RBF SVM & 0.05 & 0.08 \\
\bottomrule
\end{tabular}
\end{table}

The results highlight a significant difference between the often-reported predictive performance metrics, and explainability metrics. This behaviour explains why multicollinearity has been overlooked in the IDS benchmark datasets.
While performance is not significantly impacted by multicollinearity, explainability metrics are highly sensitive. If datasets are used without multicollinearity controls, feature attributions are unstable and experimental results are not reproducible.

Further, these results validate Theorem~\ref{math:theorim} and underline the practical value of our proposed \textit{Fragility Score}.

\section{Mitigating Explainability Fragility Methodology}
\label{sec:mitigation}

Earlier sections established the generalised effect of multicollinearity-induced fragility beyond the choice of modelling approach. However, it is not always possible or suitable to apply appropriate filtering for multicollinearity directly at the dataset level. This section aims to propose two novel methods for explainability fragility mitigation at the modelling step of AI/ML pipelines. The two approaches are not interdependent and are designed to be applied either in already trained models or while training. 

\subsection{Post-hoc CAA-Filtering method}

For models that have been already trained, it is not easy to evaluate datasets, apply VIF-related feature filtering and retrain, especially for highly dimensional datasets and complex models where training requires significant resources. In this case it is possible to mitigate fragility post-hoc through the proposed Collinearity-Aware Attribution (CAA) Filter presented in Algorithm~\ref {alg:caa_filter}. 

The CAA-Filter is fundamentally an approach to grouping interchangeable SHAP explanation post-hoc to avoid misleading explanations. 
The algorithm groups highly correlated (e.g., Pearson $|\rho| > 0.85$ or shared $\mathrm{VIF} > 10$) features and then aggregates SHAP values for each group. The aggregation method is selected by the user and can be either (i) the max absolute SHAP for the dominant feature, (ii) the mean SHAP to represent the distributed contribution, or (iii) the sum to represent the total group effect. The algorithm then replaces individual SHAP values with the group-level calculation. Hence, reducing explanation noise, eliminating fragility effects, and improving interpretability and reproducibility. 

This approach can be retrofitted to existing models and even existing results of previous explanation runs. As a result, and by definition, it does not affect model performance metrics. The limitation it inherently introduces is the fact that explanations cannot be uniquely attributed to a single feature, but rather a cluster of features. Additionally, it remains possible that some instability is present, especially when the max absolute SHAP filtering approach is selected. 

\begin{algorithm}
\caption{Collinearity-Aware Attribution Filter (CAA-Filter)}
\label{alg:caa_filter}
\begin{algorithmic}[1]
\REQUIRE SHAP matrix $S \in \mathbb{R}^{n \times d}$, feature matrix $X \in \mathbb{R}^{n \times d}$,
correlation threshold $\rho_{\text{thresh}}$, aggregation method $\mathcal{A} \in \{\text{mean}, \text{max}, \text{sum}\}$
\ENSURE Filtered SHAP matrix $S_{\text{filtered}}$, feature cluster mapping $\mathcal{C}$

\STATE Compute absolute correlation matrix $R \leftarrow |\mathrm{corr}(X)|$
\STATE Initialize empty cluster set $\mathcal{C} \leftarrow \emptyset$
\STATE Initialize unassigned feature set $\mathcal{U} \leftarrow \{1, \dots, d\}$

\WHILE{$\mathcal{U} \neq \emptyset$}
    \STATE Select feature $i \in \mathcal{U}$
    \STATE Initialize cluster $G \leftarrow \{i\}$
    \FORALL{$j \in \mathcal{U},\; j > i$}
        \IF{$R_{ij} > \rho_{\text{thresh}}$}
            \STATE $G \leftarrow G \cup \{j\}$
        \ENDIF
    \ENDFOR
    \STATE $\mathcal{C} \leftarrow \mathcal{C} \cup \{G\}$
    \STATE $\mathcal{U} \leftarrow \mathcal{U} \setminus G$
\ENDWHILE

\STATE Initialize $S_{\text{filtered}} \leftarrow [\;]$

\FORALL{clusters $G \in \mathcal{C}$}
    \STATE Extract submatrix $S_G \leftarrow S[:, G]$
    \IF{$\mathcal{A} = \text{mean}$}
        \STATE $s_G \leftarrow \mathrm{mean}(S_G,\; \text{axis}=1)$
    \ELSIF{$\mathcal{A} = \text{max}$}
        \STATE $s_G \leftarrow \max(|S_G|,\; \text{axis}=1)\cdot
        \mathrm{sign}\!\left(S_G[\arg\max |S_G|]\right)$
    \ELSIF{$\mathcal{A} = \text{sum}$}
        \STATE $s_G \leftarrow \sum S_G,\; \text{axis}=1$
    \ENDIF
    \STATE Append $s_G$ to $S_{\text{filtered}}$
\ENDFOR

\RETURN $S_{\text{filtered}}, \mathcal{C}$
\end{algorithmic}
\end{algorithm}

\subsection{Training-time SHARP regularisation method}

To overcome the limitations of the CAA-Filter, a formal method for the introduction of fragility-aware training is proposed. The SHAP Stability Regulariser (SHARP) is a novel method to penalise fragile attributions during model training (Algorithm~\ref{alg:sharp}). This leads to models that are producing stable attributions even if the dataset exhibits multicollinearity. SHARP can be interpreted as an attribution-space analogue of classical regularisation (L1/L2 penalties), which constrain parameter magnitude, while SHARP constrains attribution variance under resampling perturbations. This introduces a second-order geometric constraint on the model’s explanation manifold, transforming interpretability stability into a first-class optimisation objective. A penalty can be added to the loss function, capturing attribution instability per batch as presented in Equation~8. This penalty is based on the Fragility Score presented in Section~\ref{sec:fragility}, where $\lambda$ is a tunable hyperparameter ranging from $0.1$ to $1.0$.

\begin{equation}
\mathcal{L}_{\text{total}}
=
\mathcal{L}_{\text{base}}
+
\lambda
\cdot
\frac{1}{|B|}
\sum_{x_i \in B}
\text{Fragility}(x_i)
\end{equation}

\noindent\textbf{Remark}. The SHARP objective transforms explainability stability from a post-hoc diagnostic metric into a training-time optimisation target. As $\lambda$ increases, the optimisation prioritises attribution consistency across bootstrap perturbations, introducing a stability-performance trade-off analogous to classical regularisation schemes.

The added penalty in the loss function enforces the selection of representations yielding consistent SHAP explanations. As the fragility score can be applied to other explainability methods (e.g., LIME), it is further possible to generalise this approach to any attribution-based explanation methodology.

The application of SHARP does not disrupt the usual model training. The method adds the computation of SHAP over a test slice (batch) every $n$ epochs, thereby computing and adding the fragility-score penalty to the updated weights. Hence, this method further improves robustness while maintaining predictive performance.

\begin{algorithm}
\caption{SHAP Stability Regularizer (SHARP)}
\label{alg:sharp}
\begin{algorithmic}[1]
\REQUIRE Training data $(X_{\text{train}}, y_{\text{train}})$, model $f(\cdot;\theta)$,
learning rate $\eta$, regularization weight $\lambda$, SHAP estimator $\mathrm{Explainer}(\cdot)$,
batch size $B$, fragility interval $k$
\ENSURE Trained model parameters $\theta$

\STATE Initialize model parameters $\theta$
\FOR{epoch $= 1$ to $N_{\text{epochs}}$}
    \FORALL{mini-batches $B = \{x_i\} \subset X_{\text{train}}$}
        \STATE Compute predictions $\hat{y} \leftarrow f(B; \theta)$
        \STATE Compute base loss $L_{\text{base}} \leftarrow \mathcal{L}(y_B, \hat{y})$
        
        \IF{$\text{epoch} \bmod k = 0$}
            \STATE Sample bootstrap batch $B' \leftarrow \mathrm{Resample}(B)$
            \STATE Fit temporary parameters $\theta' \leftarrow \mathrm{Train}(f, B', y_B)$
            \STATE Compute SHAP values $\{\phi_i\} \leftarrow \mathrm{Explainer}(f(\cdot;\theta), B)$
            \STATE Compute SHAP values $\{\phi_i'\} \leftarrow \mathrm{Explainer}(f(\cdot;\theta'), B)$
            \STATE Compute batch fragility penalty
            \[
                \mathcal{F}_B
                \leftarrow
                \frac{1}{|B|}
                \sum_{x_i \in B}
                \text{Fragility}(x_i)
            \]
        \ELSE
            \STATE $\mathcal{F}_B \leftarrow 0$
        \ENDIF
        
        \STATE Compute total loss
        \[
            \mathcal{L}_{\text{total}}
            \leftarrow
            \mathcal{L}_{\text{base}}
            +
            \lambda \cdot \mathcal{F}_B
        \]
        \STATE Update parameters $\theta \leftarrow \theta - \eta \nabla_{\theta} \mathcal{L}_{\text{total}}$
    \ENDFOR
\ENDFOR

\RETURN $\theta$
\end{algorithmic}
\end{algorithm}

The approach is model-agnostic. It is directly compatible with linear regressors, tree models and is applicable to both differentiatable and non-differentiatable models. It does not rely on VIF dataset pruning or feature dimentionality reduction. The SHARP method operationalises the fragility score and theoretical grounding presented in earlier sections to establish explainability stability even in the presence of multicollinear datasets. Thus this method moves beyond benchmarking to introduce a new learning objective for the AI/ML model. As this approach can be applied to both convex regularisation (e.g. LR) and differentiable proxy regularisation (e.g. MLP NN) it is versatile and widely applicable. Especially for convex regularisation the results would be deterministic and converging. For differentiable proxy regularisation the First-order Taylor attribution can be utilised to provide local explanation behaviour representations that are fully differentiable and is a well-established approach for gradient-based attribution during training \cite{rawal2025evaluating}. 

Unlike VIF-based feature pruning, SHARP does not alter the dataset or remove predictors, and unlike CAA-Filtering it does not group attributions. Instead, it enforces stability constraints directly in parameter space, enabling attribution consistency even in structurally redundant datasets.

\subsection{Experimental Evaluation}

The proposed CAA-Filter and SHARP regularisation methods are, by definition, model-agnostic. Both approaches are suitable for any supervised modelling approach attempting to produce explainability outputs using SHAP or LIME inspired approached based on attributions. This section aims to prove this argument by providing experimental evaluation on representative models. 

For CAA-Filter, XGBoost is selected as a representative higher-order modelling approach, which showed a significant improvement when VIF-based pruning was applied to the dataset. XGBoost produces clean SHAP values with standard \texttt{shap.TreeExplainer} or \texttt{shap.LinearExplainer}, similarly to LR or RF. Usually, for these methods, feature importance is interpretable and sparse. Hence, attribution clusters are suitable, while speed of training and clustering can be additional benefits. As mentioned earlier, predictive performance results are not affected by this method, so only explainability results will be reported. Two instability measures are provided, both using Kendall's $\tau$ to evaluate the ranking order (in)stability. The first evaluates the instability of feature importance for the top 50 ranked features. The second evaluates the fragility instability based on the Fragility Score ranking of the top 20 most fragile features. 

For SHARP, both LR and MLP NN approaches are selected as representatives with significant VIF-pruning benefits and relative complexity. Both methods have relatively simple loss functions or well-established proxies, which result in quicker SHARP calculations of the regularisation term while they offer consistent SHAP attributions. Furthermore, $\lambda$-ablation instability is measured to establish the monotonic behaviour of the hyperparameter. 

\subsubsection{CAA-Filter applied to XGBoost}

This experiment followed the same setup as in Section~\ref{sec:results}. The exact same code was executed for the control and hypothesis cases to provide reproducible results. After the XGBoost model is trained the CAA-filter is applied, and Kendall's $\tau$ is calculated over the Top-50 features ranked by feature importance. Table~\ref{tab:caa} reports the average of 10 repetitions. The resulting improvement in feature importance ranking is consistently better than XGBoost trained on the control dataset, but lower than the model trained on the hypothesis set. Similarly, the median Kendall's $tau$ of the top-50 important features improved by approximately by 0.2 points when CAA-Filtering was applied compared to the Control dataset, while it improved by a further 0.2 points when the Hypothesis dataset was used. This conforms with the expectation of dataset cleaning being the ultimate solution. However, CAA-Filtering already provides a significant improvement of the order of 50\%. Hence, making it a viable retrofitting method when retraining is not an option. 

\begin{table}[H]
\centering
\caption{Comparative analysis results for 10 repetitions of training an XGBoost model with the full dataset (control), the VIF-based pruned (hypothesis) dataset, and applying the CAA-Filtering method to the full dataset (CAA-Filtering).}
\label{tab:caa}
\begin{tabular}{lccc}
\toprule
\textbf{Metric} & \textbf{$\tau_c$} & \textbf{$\tau_{CAA}$} & \textbf{$\tau_h$} \\
\midrule
Kendall's $\tau$ of top-50 important features & 0.44 & 0.53 & 0.68 \\
Kendall's $\tau$ of top-20 fragile features & 0.20 & -- & 0.07 \\
\bottomrule
\end{tabular}
\end{table}

Additionally, the table presents the results for a secondary Kendall's $\tau$ calculation over the Top-20 features ranked by fragility score. This metric is also averaged over the 10 repetitions, and it is counter-analogous to the stability of the model, meaning that lower values are indicative of improved stability. The results for the CAA-filter cannot be calculated for this secondary metric, as the grouping of the features would make the comparison inconsistent. However, the hypothesis set produces more stable explanations with lower fragility scores. 

\subsubsection{SHARP applied to LR}

For reproducibility, the setup was as described in Section~\ref{sec:results}.
For the evaluation of explainability and performance metrics, the hyperparameter $\lambda$ was initially fixed at $0.5$ across 10 repetitions. The implementation of the model was based on \texttt{PyTorch} and \texttt{numpy.Linear} with the optimiser \texttt{Adam} and \texttt{BCELoss} running over 10 epochs with batch size 64 and applying the \texttt{shap.Explainer} to each batch to calculate the fragility penalty. 

The results are presented in Tables~\ref{tab:sharp_lr_explainability} and \ref{tab:sharp_lr_performance} respectively. As expected, Kendall's $\tau$ reaches the highest possible value when SHARP is used, as it is central to the optimisation goal. However, a surprising result was that performance metrics for LR were also optimised with SHARP and indeed exceeded the results of the VIF pruning approach. The results demonstrate perfect interpretation rank stability across bootstraps, elimination of explanation fragility while achieving stronger generalisation than VIF pruning with more interpretable and also more accurate modelling.
Not only is there no accuracy sacrifice to achieve better explainability, but infact there is a performance boost for every single performance metric. 

\begin{table}[H]
\centering
\caption{Comparative analysis results for 10 repetitions of training an LR model with the full dataset ($_c$), the VIF-based pruned ($_h$) dataset, and applying the SHARP method to the full dataset ($_s$).}
\label{tab:sharp_lr_explainability}
\begin{tabular}{lccc}
\toprule
\textbf{Metric} & \textbf{$\tau_c$} & \textbf{$\tau_s$} & \textbf{$\tau_h$} \\
\midrule
Kendall's $\tau$ of top-50 important features & 0.50 & \textbf{0.62} & 0.84 \\
Kendall's $\tau$ of top-20 fragile features & 0.90 & \textbf{0.4} & 0.30 \\
\bottomrule
\end{tabular}
\end{table}

\begin{table}[H]
\centering
\caption{Averaged performance results for 10 repetitions of training an LR model with the full dataset ($_c$), the VIF-based pruned ($_h$) dataset, and applying the SHARP method to the full dataset ($_s$).}
\label{tab:sharp_lr_performance}
\begin{tabular}{lccccc}
\toprule
\textbf{Model} & \textbf{Accuracy} & \textbf{Precision} & \textbf{Recall} & \textbf{F1-score} & \textbf{ROC AUC} \\
\midrule
LR$_c$ & 80.7\% & 78.8\% & 95.6\% & 86.4\% & 0.750 \\
LR$_s$ & \textbf{89.8}\% & \textbf{87.6}\% & \textbf{97.9}\% & \textbf{92.4}\% & \textbf{0.976} \\
LR$_h$ & 78.9\% & 77.8\% & 93.8\% & 85.1\% & 0.732 \\
\bottomrule
\end{tabular}
\end{table}

To formally validate the controllability of explainability stability, $\lambda$-ablation evaluation was performed for $\lambda \in \{0, 0.01, 0.1, 1.0, 10.0\}$. When $\lambda = 0$, SHARP reduces to standard empirical risk minimisation, ensuring backward compatibility with existing training pipelines. Results are presented in Fig.~\ref{fig:lr_abl}. From the results, it can be observed that stability increases smoothly with $\lambda$ while fragility decreases in a similar manner. This remains linear for $\lambda \in {0,1}$, however a more dramatic change happens for values higher than $1$. Along the same trend, performance degrades only slightly as $\lambda \to 1$ and decays exponentially for higher values. Results demonstrate a monotonic relationship between $\lambda$ and attribution stability, confirming that SHARP enables deterministic regulation of explainability robustness. Notably, SHARP does not impose a stability-performance trade-off in the linear case; instead, it improves both predictive generalisation and attribution robustness, suggesting that fragility regularisation may implicitly mitigate overfitting under multicollinearity.

\begin{figure}
    \centering
    \includegraphics[width=\linewidth]{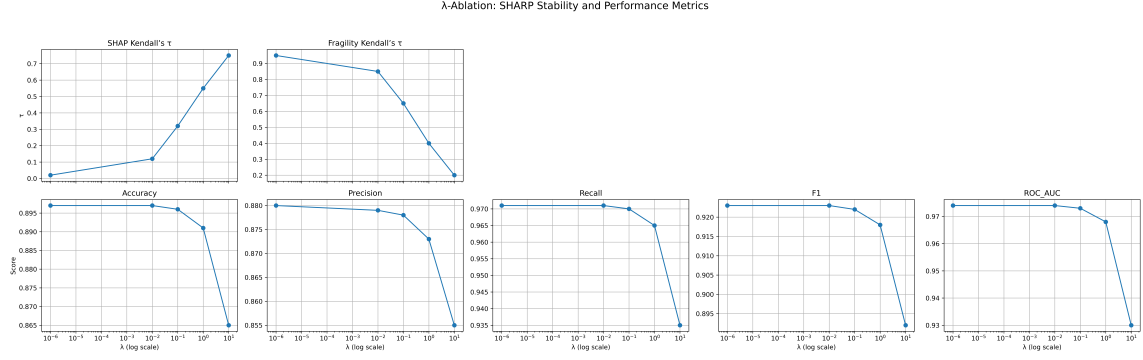}
    \caption{$\lambda$-Ablation results for values from 0 to 10 increasing by powers of 10 for the LR model}
    \label{fig:lr_abl}
\end{figure}

\subsubsection{SHARP applied to MLP NN}

For the evaluation of explainability and performance metrics, the hyperparameter $\lambda$ was initially fixed to $0.01$ for 10 repetitions, following the setup described in Section~\ref{sec:results}. The value of $\lambda$ was set to 0.5. In this case, the \texttt{shap.KernelExplainer} was used. The results are presented in Tables~\ref{tab:sharp_mlp_explainability} and \ref{tab:sharp_mlp_performance} respectively. To our knowledge, this constitutes the first application of training-time explainability regularisation in IDS benchmarking, demonstrating that even non-linear neural architectures can be constrained to produce stable attributions under multicollinear feature regimes.

The experimental stability results demonstrate once more that the dataset VIF-based pruning provides the highest or ideal stability. However, the results for SHARP are much closer to the ideal than the respective CAA-Filtering results. This verifies that the method provides improvements compared to the control case and can approximate the behaviour of dataset pruning without sacrificing any signal. At the same time, performance is largely unchanged, with marginal improvement, which also mitigates any performance loss that VIF-based pruning reports. Thus, SHARP is the best approach to provide high performance and explainability stability. 

\begin{table}[H]
\centering
\caption{Comparative analysis results for 10 repetitions of training an MLP NN model with the full dataset ($_c$), the VIF-based pruned ($_h$) dataset, and applying the SHARP method to the full dataset ($_s$).}
\label{tab:sharp_mlp_explainability}
\begin{tabular}{lccc}
\toprule
\textbf{Metric} & \textbf{$\tau_c$} & \textbf{$\tau_s$} & \textbf{$\tau_h$} \\
\midrule
Kendall's $\tau$ of top-50 important features & 0.82 & \textbf{0.89} & 0.85 \\
Kendall's $\tau$ of top-20 fragile features & 0.45 & \textbf{0.36} & 0.35 \\
\bottomrule
\end{tabular}
\end{table}

\begin{table}[H]
\centering
\caption{Averaged performance results for 10 repetitions of training an MLP NN model with the full dataset ($_c$), the VIF-based pruned ($_h$) dataset, and applying the SHARP method to the full dataset ($_s$).}
\label{tab:sharp_mlp_performance}
\begin{tabular}{lccccc}
\toprule
\textbf{Model} & \textbf{Accuracy} & \textbf{Precision} & \textbf{Recall} & \textbf{F1-score} & \textbf{ROC AUC} \\
\midrule
MLP NN$_c$ & 92.8\% & 94.3\% & 94.4\% & 94.4\% & 0.922 \\
MLP NN$_s$ & \textbf{92.8}\% & \textbf{95.0}\% & 93.5\% & 94.2\% & \textbf{0.985} \\
MLP NN$_h$ & 88.8\% & 90.3\% & 92.4\% & 91.3\% & 0.874 \\
\bottomrule
\end{tabular}
\end{table}

Similar to the case of LR, $\lambda$-ablation evaluation was performed. In this case, the values tested were $\lambda \in \{0, 0.01, 0.1, 1.0, 10.0\}$. Results are presented in Fig.~\ref{fig:mlp_abl}. Especially in the case of MLP NNs, multicollinearity directly destabilises gradient flow, and SHAP instability propagates through non-linear feature clustering. As a result, this experiment supports this work's strongest contribution: the model-agnostic, theoretically established fragility penalty as a monotonic and deterministic training-time contributor to explainability stability. Additionally, SHARP regularisation improves explanation stability while preserving predictive performance. In the case of MLP NN, the performance even picks when SHARP is applied compared to either the control or the hypothesis dataset. More generally, for moderate $\lambda$ values, the gain in explanation stability is significant compared to any performance degradation. For larger values of $\lambda$ the trade-offs are controllable. In non-linear models, where gradient flow instability amplifies multicollinearity effects, SHARP regularisation provides measurable and controllable stabilisation of SHAP attributions.

\begin{figure}
    \centering
    \includegraphics[width=\linewidth]{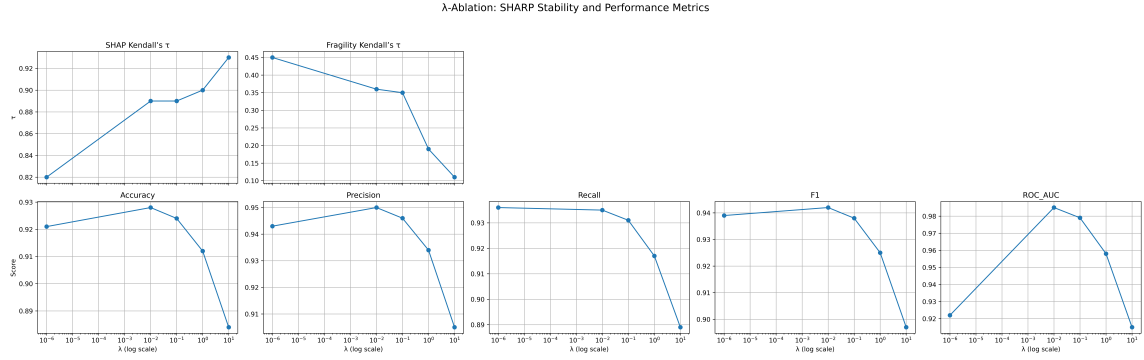}
    \caption{$\lambda$-Ablation results for values from 0 to 10 increasing by powers of 10 for the MLP NN model}
    \label{fig:mlp_abl}
\end{figure}

These results demonstrate that fragility regularisation extends beyond convex models as attribution stability improves consistently as $\lambda$ increases for MLP NNs. This further supports the theoretical positioning of the explanation fragility score as a structural regulariser on representation learning, introducing a second-order constraint on the model’s attribution geometry. This extends the classical regularisation, which constrains parameter magnitude, and introduces fragility scores to constrain attribution variance under resampling perturbations, thus establishing explainability stability as a first-class optimisation objective. 

\section{Discussion and Recommendations}
\label{sec:discussion}

\subsection{Turning tides in Explainability Positioning}

Prior explainability research has largely treated instability as a diagnostic artefact revealed post-hoc and often without sufficient reproducibility studies. In contrast, SHARP introduces stability as an explicit optimisation target, allowing practitioners to enforce attribution robustness during training rather than merely auditing it afterwards. This shift reframes explainability from a passive reporting mechanism to an actively controlled property of the model. 

The three mitigation methodologies operate at different layers of the AI pipeline, enabling practitioners to introduce what is feasible for their organisation at various stages of the development process. VIF-pruning operates at the dataset level and requires pre-processing and retraining. CAA-Filtering operates at the explanation layer, fully decoupled from the training process and applicable in mature models with minimal intervention. SHARP operates at the optimisation level with no dataset implications but with retraining requirements with optimised results and some times even better performance results with an impact on computational resources.

\subsection{Reproducibility and Benchmark Validity}

The theoretical analysis and experimental results clearly demonstrate that the effects of multicollinearity exist even in widely adopted benchmark datasets such as UNSW-NB15. Predictive metrics are largely insensitive to its presence. However, the interpretability and reproducibility of XAI methods are significantly compromised.  

The fragility analysis confirms the validity of the Theorem of Multicollinearity-Induced Experimental Fragility. In this work, it is proven both theoretically and empirically that highly multicollinear features produce unstable attributions, which in turn form the basis of explainability claims. 

The resulting implications for IDS research are profound when XAI methods are used for explanations or feature rankings are presented without prior dataset processing to remove multicollinearity beyond simple correlation-based pre-processing. Thus, any unique causal contributions of the top features are unreliable. Hence, the reproducibility of IDS research is materially affected. \textbf{Such claims must be regarded with caution unless accompanied by diagnostic evidence.} 

Most importantly, this emission in dataset auditing is a significant gap in the established protocols for benchmark validity. Trustworthy AI requires accountability, and thus auditing of dataset statistical properties becomes a mandatory requirement for compliance and reproducibility. Our work proposes a set of audit metrics which should as minimum include VIF distributions, correlation matrices, and fragility scores as well as class imbalances.

\subsection{Explainability and Regulatory Compliance}

With the growing regulatory body around Trustworthy AI, explainability is a significant concern for both academic and industrial AI deployments. As a result XAI approaches are becoming mainstream. As the regulatory frameworks such EU AI Act and others are particularly strict for critical applications, explanations and auditability become especially important in the domains of finance, healthcare and cybersecurity. As demonstrated by Theorem~\ref{math:theorim}, the established XAI approaches such as SHAP and LIME are predominantly based in feature attributions and thus unstable under multicollinearity undermining the credibility of interpretability or explainability claims. 

Furthermore, by enabling controllable explanation stability without altering predictive performance, SHARP supports reproducibility, auditability, and robustness requirements outlined in emerging AI governance frameworks.

\subsection{Generalization Beyond UNSW-NB15}

The proposed Theorem and Explainability Fragility Score are generalised and applicable to any method that bases intrpretations or explanations on feature attributions. Thus the work can generalise to any dataset beyond UNSW-NB15 and even beyond the domain of IDS research. Within this domain popular benchmarks include CICIDS2017, Bot-IoT, and TON\_IoT. All of them provide similar features to UNSW-NB15 and as a result share similar structural redundancies. For example, features from flow statistics and protocol attributes are extracted and added to the dataset. Thus, the proposed Fragility Score metric as well as the related Kendall's $\tau$ measurements can be used as a general diagnostic tool applicable to any cybersecurity benchmark dataset. Similarly, highly dimensional datasets with engineered features are prevalent in a variety of domains (e.g., healthcare, finance) where public datasets are widely reused without multicollinearity audits.

\subsection{Recommendations for Practitioners}

Proposed new Theorems or metrics often face the challenge of relevance to applied research and more importantly practical implementations in industrial settings. This section aims to address this challenge by presenting a set of four recommendations for practical integration. The aim is to ensure methodological rigour and trustworthy deployment of models trained on UNSW-NB15 and similar cybersecurity benchmarks.

\subsubsection{Data pipelines}

Data pipelines are well established in current SaaS offerings for AI/ML often focusing on data pre-processing. Thus, including multicollinearity diagnostics such us the calculation of fragility scores can be embedded in such pipelines. It is frequently possible to enable pipeline checks based on reported measures, and as a result, automate the enforcement of multicollinearity-free datasets even in the presence of continuously evolving incoming data and data features. Such approaches can be included in data pipelines with small overhead while ensuring explainability and interpretability for robust and reproducible results.

\subsubsection{Dataset Diagnostics}

\begin{itemize}
  \item \textbf{Report multicollinearity metrics.} In recent years, reporting class imbalance statistics has become standard practice. This reporting should be extended to include VIF and correlation matrices for benchmark datasets.
  
  \item \textbf{Flag high-risk features.} To enable automated control of multicollinearity, highly multicollinear ($\text{VIF} > 10$) or correlated ($|\rho| > 0.85$) features should be automatically flagged as risks.
  
\end{itemize}

\subsubsection{Feature Engineering Practices}

\begin{itemize}
  \item \textbf{Remove or aggregate redundant features.} Flagged features can be automatically reduced or aggregated based on SEM or similar approaches.
  
  \item \textbf{Apply dimensionality reduction.} Several approaches are often present in the SaaS toolkit for AI/ML pipelines, such as PCA or autoencoders. However, their output should again be examined under the dataset diagnostics proposed above and flagged if necessary. 
  
  \item \textbf{Use regularisation.} Similarly, regularisation can achieve a reduction in redundancy with methods such as Lasso or ElasticNet. However, the output should again pass the dataset diagnostics threshold checks.

\end{itemize}

\subsubsection{Explainability and Stability Checks}

\begin{itemize}

  \item \textbf{Compute Fragility Scores.} To ensure the final outcome is identifiable and reproducible, fragility scores should be computed for the trained model before reporting SHAP or LIME feature rankings. Resampling or feature perturbation approaches can be used to iterate. 

  \item \textbf{Apply training-time fragility regularisation (SHARP) or post-hoc CAA-Filtering}. Where retraining is feasible, integrate fragility-aware penalties directly into the loss function to enforce attribution stability prior to deployment. Otherwise, prefer the CAA-Filtering method, which provides flexibile integration without re-training and already improves the stability of explanations.
    
  \item \textbf{Stability as a validity criterion.} If possible, apply a threshold to fragility scores linked to the deployment pipeline. If features with high fragility are mandatory, then, as a minimum measure, high Fragility Score features must not contribute to the interpretation or causality of the prediction.

  \item \textbf{Trustworthy AI compliance.} By introducing continuous monitoring of fragility, it is possible to audit model performance throughout its lifetime, and thus ensure continuous compliance with the EU AI Act. The model explanations can be demonstrably reliable at any point in time.
  
\end{itemize}

\section{Conclusion and Future Work}

This paper presented the theoretical framework for multicolinearity-induced explainability fragility, along with empirical evaluation using the UNSW-NB15 benchmark dataset for AI-based IDS. The case study demonstrates, through a systematic analysis, the severity of multicollinearity in established datasets. The results demonstrate VIF values and fragility scores order of magnitude beyond standard thresholds. Overall model explainability fragility was further highlighted through Kendall's $\tau$ calculations over feature importance ranks while predictive performance remains largely unaffected. Through mathematical proofs, the paper demonstrated the effect of multicollinearity over standard XAI approaches based on attributions, concluding that XAI claims can be unreliable, non-identifiable and thus not reproducible. The work extends the state-of-the-art by contributing (i) a formal theorem to quantify the experimental validity in XAI for IDS research, (ii) an audit of an established benchmark dataset identifying high-risk features, (iii) a formally defined fragility score quantifying attribution instability applicable to a variety of XAI methods, and (iv) reframing explainability from a passive post-hoc interpretation tool to an actively optimised training objective, establishing stability as a measurable and controllable property of AI systems.
 
Together, these findings challenge the validity of prior IDS studies that report ``top features'' or explanation-driven insights. We argue that future IDS benchmarking must report not only predictive performance but also explainability stability metrics, and where instability is detected, apply structured mitigation such as VIF control, CAA-Filtering, or SHARP regularisation. Finally, the paper concludes with a set of actionable recommendations for practitioners for reproducible and trustworthy evaluation. 

\subsection{Future work}

There are several avenues to extend this work further and enable rigorous IDS evaluation supporting explainable, reliable, and compliant solutions. The following comprises a list of possible further research directions: 
\begin{itemize}
  \item Replicate the experimental steps of the multicollinearity audit for further IDS benchmarks such as CICIDS2017, Bot-IoT, and TON\_IoT, thus evaluating the generalisation claims;
  \item Explore other causal inference approaches which could be linked to dataset pre-processing and data pipeline monitoring;
  \item Integrate fragility auditing into open-source benchmarking toolkits, achieving standardisation of reproducibility protocols for IDS research;
  \item Assess the broader impact of multicollinearity on proposed fairness metrics and adversarial robustness;
  \item Combine CAA-Filtering and SHARP at CAA cluster level to identify further potential enhancements; 
  \item Extend SHARP to support online monitoring. 
\end{itemize}

%

\bibliographystyle{ACM-Reference-Format}
\bibliography{sample-base}

\appendix

\section{Theorem Proof}
\label{sec:proof}
\begin{proof}[Theorem derivation]
\textbf{Setting.}
Let $X\in\mathbb{R}^{n\times p}$ be the feature matrix, $y\in\mathbb{R}^n$,
and consider the homoscedastic linear model
$y = X\beta + \varepsilon$, where the mean-noise $\mathbb{E}[\varepsilon]=0$, providing unbiased results, and the variance of error is constant
$\operatorname{Var}(\varepsilon)=\sigma^2 I_n$.
Set $x_i$ as the feature of the $i$th feature column and set $X_{-i}$ for all columns except $x_i$.
In the case of linear models, then we can define the uncentred second moment of mean-centred $x_i$ scaled by $n$ as $S_{ii}=\sum_{k=1}^n x_{ki}^2 = n\,\operatorname{Var}(x_i)$ if and only if $\bar{x_i} = 0$. Considering linear regression, $S_{ii}$ will appear as the diagonal element of the Gram matrix $(X^\top X)$ leading to $S_{ii} = (X^\top X)_{ii}$. Additionally, let
$R_i^2$ as the $R^2$ from regressing $x_i$ on $X_{-i}$.
Finally, let variance inflation factor be $\text{VIF}(x_i)=\frac{1}{1-R_i^2}$.

\medskip
\noindent\textbf{Part (A): Non-identifiability as $\text{VIF}(x_i)\to\infty$.}
If $R_i^2=1$, then $x_i$ lies in the span of $X_{-i}$; there exists $\alpha\in\mathbb{R}^{p-1}$ such that perfect linear predictability exists as
$x_i=X_{-i}\alpha$. This means that the design matrix is rank-deficient as the columns are linearly dependent. Then, there exists a nontrivial non-zero vector $\gamma\neq 0$ which is the null-space of X $X\gamma=x_{i}-X_{-i}\alpha=0$ where: 

\begin{equation} 
\gamma_{j}=
\begin{cases}
1 & \text{if } j = i \\
-\alpha_j & \text{if } j \ne i
\end{cases}
\end{equation}

Thus any $t\in\mathbb{R}$ where by a coefficient vector $\beta$ provides an identical result as the coefficeint vector $\beta^\prime$ defined as:

\begin{equation}
\beta^\prime \;=\; \beta + t \begin{bmatrix} 1 \\ -\alpha \end{bmatrix}_i = \beta + t \gamma
\quad\Longrightarrow\quad
X\beta^\prime \;=\; X\beta + t\,(x_i - X_{-i}\alpha) \;=\; X\beta.
\end{equation}

Hence, the fitted values and resulting predictions and losses are unchanged along a nontrivial subspace of coefficients. As a result, any linear attribution method based on linear coefficients is therefore \emph{non-identifiable}. This applies to some of the most popular methods for computing feature contributions and thus explainability metrics, such as regression coefficients, linear SHAP, and local linear LIME.
As $R_i^2\to 1$ or $\text{VIF}(x_i)\to\infty$, the Gram matrix $X^\top X$ becomes ill-conditioned and the same non-uniqueness emerges in the limit, thus proving the first part of the theorem.

\medskip
\noindent\textbf{Part (B): Variance lower bound scaling with VIF for linear models.}
Let us assume full rank, then, the Ordinary Least Squares (OLS) estimator is $\hat\beta=(X^\top X)^{-1}X^\top y$.
Thus, under Gauss-Markov assumptions set out above, the covariance of the OLS estimator  is:

\begin{equation}
\operatorname{Var}(\hat\beta)=\sigma^2 (X^\top X)^{-1}.
\end{equation}

Let the $i$-th diagonal element of the inverse of the Gram matrix $(X^\top X)^{-1}_{ii}$ measure the sensitivity in estimating the coefficient $\beta_{i}$. Using the well-known Schur complement identity:

\begin{equation}
(X^\top X)^{-1}_{ii}
\;=\; \frac{1}{S_{ii}(1-R_i^2)}
\;=\; \frac{\text{VIF}(x_i)}{S_{ii}}
\;=\; \frac{\text{VIF}(x_i)}{n\,\text{Var}(x_i)}.
\end{equation}

Hence:

\begin{equation}
\text{Var}(\hat\beta_i) \;=\; \sigma^2 (X^\top X)^{-1}_{ii}
\;=\; \frac{\sigma^2}{n\,\text{Var}(x_i)}\,\text{VIF}(x_i).
\end{equation}

Consider the attribution for a \emph{fixed test point} $x^\ast$ under a linear model.
For many linear explanation schemes such as linear SHAP with interventional semantics on linear models,
the $i$-th attribution takes the form $\phi_i(x^\ast)=\hat\beta_i\,(x^\ast_i-\mu_i)$, where $\mu_i$ is a fixed baseline. In the setting presented above this baseline is  $\mathbb{E}[x_i]$.
Since $x^\ast$ is fixed when we bootstrap the training set, randomness only comes from $\hat\beta_i$.
Therefore, it can be derived that:

\begin{equation}
\text{Var}\!\big(\phi_i(x^\ast)\big)
\;=\; (x^\ast_i-\mu_i)^2 \,\text{Var}(\hat\beta_i)
\;=\; (x^\ast_i-\mu_i)^2\;\frac{\sigma^2}{n\,\text{Var}(x_i)}\,\text{VIF}(x_i).
\end{equation}

Because $\text{VIF}(x_i)\ge 1$, we have $\text{VIF}(x_i)\ge (\text{VIF}(x_i)-1)$, and thus

\begin{equation}
\text{Var}\!\big(\phi_i(x^\ast)\big)
\;\ge\; \underbrace{\frac{(x^\ast_i-\mu_i)^2\,\sigma^2}{n\,\text{Var}(x_i)}}_{=:~c(x^\ast)>0}\;(\text{VIF}(x_i)-1).
\end{equation}

Averaging over a fixed evaluation set $\mathcal{D}_{\text{eval}}$ of test points replaces $(x^\ast_i-\mu_i)^2$ by its empirical mean $\frac{1}{|\mathcal{D}_{\text{eval}}|}\sum_{x^\ast\in\mathcal{D}_{\text{eval}}}(x^\ast_i-\mu_i)^2$, yielding the same linear-in-$\text{VIF}$ scaling with a model- and evaluation-dependent constant $c>0$, thus, proving the second part of the theorem in the case of linear models.

\medskip
\noindent\textbf{Part (C): Extension to SHAP/LIME via local linear surrogates.}
Further to the earlier parts, it is possible to expand the derivation to provide proof for model-agnostic explainable AI methods such as SHAP and LIME. These methods depend on local linear surrogates to explain predictions of any model $f$. Kernel SHAP and LIME estimate attributions by solving a weighted local linear regression around $x^\ast$:

\begin{equation}
\hat\beta(x^\ast) \;=\; \arg\min_\beta \|W^{1/2}(f(Z)-Z\beta)\|_2^2,
\quad\text{so}\quad
\text{Var}\big(\hat\beta(x^\ast)\big) \;=\; \sigma^2\,(Z^\top W Z)^{-1}.
\end{equation}

Where: $Z$ is a locally sampled design and $W$ are locality weights.
In this case it is possible to apply again the Schur-complement identity approach to $Z^\top W Z$. Hence, it can be derived that: 

\begin{equation}
\text{Var}\big(\hat\beta_i(x^\ast)\big)
\;=\; \frac{\sigma^2}{S^{(W)}_{ii}}\;\text{VIF}^{(W)}(x_i\mid x^\ast),
\end{equation}

where $S^{(W)}_{ii}$ is the weighted sum of squares of the local $i$th feature and $\text{VIF}^{(W)}$ is the \emph{local} weighted VIF induced by $(Z,W)$.
This translates to the following $i$-th local attribution:

\begin{equation}
\phi_i(x^\ast)=\hat\beta_i(x^\ast)\cdot (x^\ast_i-\mu_i)
\end{equation} 

Which, if applied to the same restriction as in Part B, obeys: 

\begin{equation}
\text{Var}\!\big(\phi_i(x^\ast)\big)
\;\ge\; c^{(W)}(x^\ast)\,\big(\text{VIF}^{(W)}(x_i\mid x^\ast)-1\big),\qquad
c^{(W)}(x^\ast)=\frac{(x^\ast_i-\mu_i)^2\,\sigma^2}{S^{(W)}_{ii}}.
\end{equation}

Thus, it is proven that explanation methods that \emph{estimate} a local linear surrogate inherit the same qualitative dependence for non-linear $f$.

\noindent\textbf{Remarks.}
(i) The constants $c$ and $c^{(W)}(x^\ast)$ depend on noise level $\sigma^2$, sample size $n$, the feature scale, and the evaluation point. They are positive and finite under standard regularity (finite second moments).\\
(ii) When $R_i^2\to 1$ and $\text{VIF}(x_i)\to\infty$, the linear system becomes singular and both coefficient estimates and linear attributions become ill-posed or non-identifiable.\\
(iii) For tree ensembles, SHAP implementations compute exact values under model-specific assumptions such as  TreeSHAP. In this case, the degeneracy of the input covariance still induces instability in conditional expectations used by some SHAP semantics, again leading to larger attribution variance.

\noindent\textbf{Conclusion.}
This formalises that multicollinearity \emph{induces experimental fragility} in feature attributions, and justifies VIF-based pruning as a principled way to improve attribution stability.
\end{proof}

\section{Raw results per tested model}
\label{sec:results}

\subsection{LR}

\begin{table}[H]
\centering
\caption{Kendall’s $\tau$ rank correlation for SHAP feature importance stability in Logistic Regression. 
Higher values indicate more stable feature rankings across bootstrap samples.}
\label{tab:kendall_lr}
\begin{tabular}{lc}
\toprule
\textbf{Model} & \textbf{$\tau$ Top-20 features} \\
\midrule
LR (Full) & 0.22 $\;\rightarrow\;$ Low stability (fragile explanations) \\
LR (VIF)  & 0.77 $\;\rightarrow\;$ High stability (robust explanations) \\
\bottomrule
\end{tabular}
\end{table}

\begin{table}[H]
\centering
\caption{Top fragile features before and after VIF pruning for Logistic Regression. 
Fragility values represent the variance-to-mean ratio of absolute SHAP values. 
Lower fragility indicates more stable feature attribution. 
Features removed due to high multicollinearity are shown as ``--''.}
\label{tab:fragility_lr}
\begin{tabular}{lcc}
\toprule
\textbf{Feature} & \textbf{Fragility LR Full} & \textbf{Fragility LR VIF} \\
\midrule
dload              & 16.42 & 12.62 \\
sbytes             & 9.33  & --    \\
rate               & 2.42  & 2.50  \\
dbytes             & 2.36  & --    \\
sload              & 0.79  & 0.83  \\
response\_body\_len & 0.13  & 0.13  \\
stcpb              & 0.03  & 0.04  \\
dtcpb              & 0.03  & 0.03  \\
sjit               & 0.03  & 0.03  \\
sinpkt             & 0.01  & --    \\
djit               & 0.0005 & 0.0005 \\
dinpkt             & 0.0001 & 0.0001 \\
spkts              & $2.3\times10^{-6}$ & -- \\
dmean              & $1.8\times10^{-6}$ & $2.0\times10^{-6}$ \\
smean              & $1.2\times10^{-6}$ & $1.0\times10^{-6}$ \\
dpkts              & $9.7\times10^{-7}$ & -- \\
sloss              & $2.8\times10^{-7}$ & -- \\
dloss              & $1.9\times10^{-7}$ & -- \\
dttl               & $1.1\times10^{-7}$ & -- \\
sttl               & $4.9\times10^{-8}$ & -- \\
\bottomrule
\end{tabular}
\end{table}

\begin{table}[ht]
\centering
\caption{Performance comparison of Logistic Regression before and after VIF pruning. 
Values are reported as percentages for Accuracy, Precision, Recall, and F1, and as proportions for ROC AUC.}
\label{tab:performance_lr}
\begin{tabular}{lcc}
\toprule
\textbf{Metric} & \textbf{LR Full} & \textbf{LR VIF} \\
\midrule
Accuracy   & 80.7\% & 78.9\% \\
Precision  & 78.8\% & 77.8\% \\
Recall     & 95.6\% & 93.8\% \\
F1-score   & 86.4\% & 85.1\% \\
ROC AUC    & 0.750  & 0.732 \\
\bottomrule
\end{tabular}
\end{table}

\subsection{XG BOOST}

\begin{table}[H]
\centering
\caption{Top fragile features before and after VIF pruning for XGBoost. 
Fragility values represent the variance-to-mean ratio of absolute SHAP values. 
Lower fragility indicates more stable feature attribution. 
Features removed due to high multicollinearity are shown as ``--''.}
\label{tab:fragility_xgb_full}
\begin{tabular}{lcc}
\toprule
\textbf{Feature} & \textbf{Fragility XGboost Full} & \textbf{Fragility XGboost VIF} \\
\midrule
proto\_ospf          & 1.95 & --\\
sttl                 & 1.68 & --\\
service\_smtp        & 0.84 & 1.29 \\
state\_RST           & 0.77 & --\\
proto\_unas          & 0.63 & --\\
sinpkt               & 0.61 & --\\
smean                & 0.60 & 0.59 \\
ct\_state\_ttl       & 0.57 & --\\
dbytes               & 0.40 & --\\
response\_body\_len  & 0.31 & 0.56 \\
dmean                & 0.29 & 0.30 \\
service\_http        & 0.29 & 0.53 \\
dloss                & 0.27 & --\\
ct\_dst\_src\_ltm    & 0.26 & --\\
ct\_dst\_sport\_ltm  & 0.26 & --\\
sbytes               & 0.25 & --\\
sloss                & 0.23 & --\\
service\_ftp-data    & 0.19 & 3.70 \\
service\_dns         & 0.19 & 0.31 \\
djit                 & 0.18 & 0.27 \\
dload                & -- & 1.91 \\
proto\_igmp           & -- & 0.72 \\
proto\_sep            & -- & 0.80 \\
\bottomrule
\end{tabular}
\end{table}

\begin{table}[H]
\centering
\caption{Kendall’s $\tau$ rank correlation for SHAP feature importance stability in XGBoost. 
Higher values indicate more stable feature rankings across bootstrap samples.}
\label{tab:kendall_xgb}
\begin{tabular}{lc}
\toprule
\textbf{Model} & \textbf{$\tau$ Top-20 features} \\
\midrule
XGB (Full) & 0.44 $\;\rightarrow\;$ Moderate stability \\
XGB (VIF)  & 0.68 $\;\rightarrow\;$ Strong stability (robust rankings) \\
\bottomrule
\end{tabular}
\end{table}

\begin{table}[H]
\centering
\caption{Performance comparison of XGBoost before and after VIF pruning. 
Values are reported as percentages for Accuracy, Precision, Recall, and F1, and as proportions for ROC AUC.}
\label{tab:performance_xgb}
\begin{tabular}{lcc}
\toprule
\textbf{Metric} & \textbf{XGB Full} & \textbf{XGB VIF} \\
\midrule
Accuracy   & 94.4\% & 93.0\% \\
Precision  & 96.1\% & 94.4\% \\
Recall     & 95.0\% & 94.6\% \\
F1-score   & 95.6\% & 94.5\% \\
ROC AUC    & 0.94   & 0.92   \\
\bottomrule
\end{tabular}
\end{table}

\subsection{Random forests}

\begin{table}[H]
\centering
\caption{Kendall’s $\tau$ rank correlation for SHAP feature importance stability in Random Forests. 
Higher values indicate more stable feature rankings across bootstrap samples.}
\label{tab:kendall_rf}
\begin{tabular}{lc}
\toprule
\textbf{Model} & $\tau$ Top-20 features \\
\midrule
\textbf{RF Full} & 0.11 $\rightarrow$ Low stability (fragile explanations) \\
\textbf{RF VIF}  & 0.51 $\rightarrow$ Moderate stability (improved explanations) \\
\bottomrule
\end{tabular}
\end{table}

\begin{table}[H]
\centering
\caption{Performance comparison of Random Forest before and after VIF pruning. 
Values are reported as percentages for Accuracy, Precision, Recall, and F1, and as proportions for ROC AUC.}
\label{tab:performance_RF}
\begin{tabular}{lcc}
\toprule
\textbf{Metric} & \textbf{RF Full} & \textbf{RF VIF} \\
\midrule
Accuracy  & 94.7\% & 92.5\% \\
Precision & 96.0\% & 94.5\% \\
Recall    & 95.6\% & 93.8\% \\
F1        & 95.8\% & 94.1\% \\
ROC AUC   & 0.943  & 0.920 \\
\bottomrule
\end{tabular}
\end{table}

\begin{table}[H]
\centering
\caption{Top fragile features before and after VIF pruning for Random Forest. Fragility values represent the variance-to-mean ratio of absolute SHAP values. 
Lower fragility indicates more stable feature attribution. 
Features removed due to high multicollinearity are shown as ``--''.}
\label{tab:fragility_rf}
\begin{tabular}{lcc}
\toprule
\textbf{Feature} & \textbf{Fragility RF Full} & \textbf{Fragility RF VIF} \\
\midrule
sttl               & 0.108 & -- \\
smean              & 0.092 & 0.109 \\
proto\_arp         & 0.083 & -- \\
state\_REQ         & 0.074 & -- \\
sinpkt             & 0.065 & -- \\
proto\_udp         & 0.035 & -- \\
synack             & 0.031 & -- \\
service\_http      & 0.027 & 0.072 \\
dbytes             & 0.024 & -- \\
ct\_srv\_dst       & 0.023 & -- \\
sbytes             & 0.023 & -- \\
sloss              & 0.021 & -- \\
ct\_dst\_src\_ltm  & 0.021 & -- \\
response\_body\_len& 0.021 & 0.028 \\
ct\_src\_dport\_ltm& 0.020 & -- \\
dloss              & 0.018 & -- \\
ct\_dst\_sport\_ltm& 0.015 & -- \\
sload              & 0.014 & 0.023 \\
ct\_srv\_src       & 0.013 & -- \\
sjit               & 0.013 & 0.030 \\
\bottomrule
\end{tabular}
\end{table}

\subsection{MLP NN}

\begin{table}[H]
\centering
\caption{Performance comparison of MPL NN before and after VIF pruning. 
Values are reported as percentages for Accuracy, Precision, Recall, and F1, and as proportions for ROC AUC.}
\label{tab:performance_mlp}
\begin{tabular}{lcc}
\toprule
\textbf{Metric} & \textbf{MLP NN Full} & \textbf{MLP NN VIF} \\
\midrule
Accuracy  & 92.8\% & 88.8\% \\
Precision & 94.3\% & 90.3\% \\
Recall    & 94.4\% & 92.4\% \\
F1        & 94.4\% & 91.3\% \\
ROC AUC   & 0.922  & 0.874 \\
\bottomrule
\end{tabular}
\end{table}

\begin{table}[H]
\centering
\caption{Kendall’s $\tau$ rank correlation for SHAP feature importance stability in MPL NN. 
Higher values indicate more stable feature rankings across bootstrap samples.}
\label{tab:kendall_mlpnn}
\begin{tabular}{lcc}
\toprule
\textbf{Model} & \textbf{Top-20 Features} & \textbf{Top-50 Features} \\
\midrule
MLP (Full) & 0.640 & 0.632 \\
MLP (VIF)  & 0.730 & 0.727 \\
\bottomrule
\end{tabular}
\end{table}

\begin{table}[H]
\centering
\caption{Top fragile features before and after VIF pruning for NPL NN. 
Fragility values represent the variance-to-mean ratio of absolute SHAP values. 
Lower fragility indicates more stable feature attribution. 
Features removed due to high multicollinearity are shown as ``--''.}
\label{tab:fragility_mlpnn}
\begin{tabular}{lcc}
\toprule
\textbf{Feature} & \textbf{Fragility MLP Full} & \textbf{Fragility MLP VIF} \\
\midrule
sload          & $6.86 \times 10^{16}$ & 0.061 \\
dmean          & $6.86 \times 10^{16}$ & 0.088 \\
service\_dns   & 1.063 & 0.037 \\
smean          & 0.906 & 0.073 \\
service\_radius & 0.524 & 0.311 \\
proto\_idrp    & 0.394 & 0.000 \\
proto\_visa    & 0.373 & 0.130 \\
proto\_idpr-cmtp & 0.310 & 0.229 \\
proto\_mfe-nsp & 0.304 & 0.215 \\
proto\_leaf-2  & 0.301 & 0.312 \\
\bottomrule
\end{tabular}
\end{table}

\subsection{Convolutional Neural Networks}

\begin{table}[H]
\centering
\caption{Performance comparison of CNN before and after VIF pruning. 
Values are reported as percentages for Accuracy, Precision, Recall, and F1, and as proportions for ROC AUC.}
\label{tab:performance_cnn}
\begin{tabular}{lcc}
\toprule
\textbf{Metric} & \textbf{CNN Full} & \textbf{CNN VIF} \\
\midrule
Accuracy  & 92.95\% & 86.15\% \\
Precision & 94.15\% & 87.90\% \\
Recall    & 94.86\% & 90.83\% \\
F1 Score  & 94.50\% & 89.34\% \\
ROC AUC   & 0.922   & 0.843   \\
\bottomrule
\end{tabular}
\end{table}

\begin{table}[H]
\centering
\caption{Kendall’s $\tau$ rank correlation for SHAP feature importance stability in CNN. 
Higher values indicate more stable feature rankings across bootstrap samples.}
\label{tab:kendall_cnn}
\begin{tabular}{lcc}
\toprule
\textbf{Model} & \textbf{Top-20 Features} & \textbf{Top-50 Features} \\
\midrule
CNN Full & 0.03 $\rightarrow$ Very Low Stability & 0.04 $\rightarrow$ Very Low Stability \\
CNN VIF  & 0.12 $\rightarrow$ Low Stability       & 0.34 $\rightarrow$ Moderate Stability \\
\bottomrule
\end{tabular}
\end{table}

\begin{table}[H]
\centering
\caption{Top fragile features before and after VIF pruning for CNN. 
Fragility values represent the variance-to-mean ratio of absolute SHAP values. 
Lower fragility indicates more stable feature attribution. 
Features removed due to high multicollinearity are shown as ``--''.}
\label{tab:fragility_cnn}
\begin{tabular}{lcc}
\toprule
\textbf{Feature} & \textbf{Fragility CNN Full} & \textbf{Fragility CNN VIF} \\
\midrule
state\_RST       & 0.435 & --     \\
proto\_vines     & 0.273 & 0.238  \\
dinpkt           & 0.244 & 0.034  \\
proto\_zero      & 0.230 & 0.236  \\
proto\_idpr      & 0.229 & 0.257  \\
proto\_aes-sp3-d & 0.222 & 0.254  \\
proto\_stp       & 0.214 & 0.244  \\
proto\_ib        & 0.213 & 0.246  \\
proto\_st2       & 0.194 & 0.285  \\
proto\_merit-inp & 0.192 & 0.203  \\
\bottomrule
\end{tabular}
\end{table}

\subsection{LSTMs}

\begin{table}[H]
\centering
\caption{Performance comparison of LSTM before and after VIF pruning. 
Values are reported as percentages for Accuracy, Precision, Recall, and F1, and as proportions for ROC AUC.}
\label{tab:performance_lstm}
\begin{tabular}{lcc}
\toprule
\textbf{Metric} & \textbf{LSTM Full} & \textbf{LSTM VIF} \\
\midrule
Accuracy  & 94.9\% & 94.7\% \\
Precision & 96.9\% & 96.1\% \\
Recall    & 95.0\% & 95.6\% \\
F1        & 95.9\% & 95.9\% \\
ROC AUC   & 0.991  & 0.991 \\
\bottomrule
\end{tabular}
\end{table}

\begin{table}[H]
    \centering
\caption{Kendall’s $\tau$ rank correlation for SHAP feature importance stability in LSTM. 
Higher values indicate more stable feature rankings across bootstrap samples.}
\label{tab:kendall_LSTMs}
\begin{tabular}{lcc}
\toprule
\textbf{Model} & \textbf{Top-20 features} & \textbf{Top-50 features} \\
\midrule
LSTM Full & $0.28$ $\rightarrow$ Moderate Stability & $0.14$ $\rightarrow$ Low Stability\\
LSTM VIF  & $0.30$ $\rightarrow$ Moderate Stability & $0.17$ $\rightarrow$ Low Stability\\
\bottomrule
\end{tabular}
\end{table}

\begin{table}[H]
\centering
\caption{Top fragile features before and after VIF pruning for LSTM. 
Fragility values represent the variance-to-mean ratio of absolute SHAP values. 
Lower fragility indicates more stable feature attribution. 
Features removed due to high multicollinearity are shown as ``--''.}
\label{tab:fragility_lstm}
\begin{tabular}{lcc}
\toprule
\textbf{Feature} & \textbf{Fragility LSTM Full} & \textbf{Fragility LSTM VIF} \\
\midrule
dttl              & $0.108$ & $0.137$ \\
id                & $0.069$ & $0.056$ \\
ct\_dst\_src\_ltm & $0.065$ & $0.076$ \\
sttl              & $0.064$ & $0.069$ \\
ct\_state\_ttl    & $0.060$ & $0.068$ \\
dmean             & $0.052$ & $0.054$ \\
service\_dns      & $0.050$ & $0.046$ \\
swin              & $0.035$ & $0.034$ \\
proto\_udp        & $0.033$ & $0.032$ \\
proto\_tcp        & $0.025$ & $0.044$ \\
proto\_unas       & $0.023$ & $0.009$ \\
dload             & $0.021$ & $0.026$ \\
state\_REQ        & $0.020$ & $0.020$ \\
dwin              & $0.019$ & $0.026$ \\
ct\_src\_dport\_ltm & $0.018$ & $0.018$ \\
synack            & $0.016$ & $0.017$ \\
ct\_dst\_ltm      & $0.015$ & $0.009$ \\
ct\_dst\_sport\_ltm & $0.014$ & $0.015$ \\
ct\_srv\_src      & $0.014$ & $0.012$ \\
smean             & $0.012$ & $0.011$ \\
\bottomrule
\end{tabular}
\end{table}

\subsection{RBF SVM}

\begin{table}[H]
    \centering
\caption{Kendall’s $\tau$ rank correlation for SHAP feature importance stability in RBF SVM. 
Higher values indicate more stable feature rankings across bootstrap samples.}
\label{tab:kendall_rbfsvm}
\begin{tabular}{lcc}
\toprule
\textbf{Model} & \textbf{Top-20 features} & \textbf{Top-50 features} \\
\midrule
SVM Full & $0.032$ $\rightarrow$ Very Low Stability & $0.053$ $\rightarrow$ Very Low Stability\\
SVM VIF & $0.045$ $\rightarrow$ Very Low Stability & $0.079$ $\rightarrow$ Very Low Stability \\
\bottomrule
\end{tabular}
\end{table}

\begin{table}[H]
\centering
\caption{Performance comparison of RBF SVM before and after VIF pruning. 
Values are reported as percentages for Accuracy, Precision, Recall, and F1, and as proportions for ROC AUC.}
\label{tab:performance_rbfsvm}
\begin{tabular}{lcc}
\toprule
\textbf{Metric} & \textbf{RBF SVM Full} & \textbf{RBF SVM VIF} \\
\midrule
Accuracy  & 89.7\% & 89.7\% \\
Precision & 86.6\% & 86.6\% \\
Recall    & 99.2\% & 99.2\% \\
F1        & 92.5\% & 92.5\% \\
ROC AUC   & 0.972  & 0.972 \\
\bottomrule
\end{tabular}
\end{table}

\begin{table}[H]
\centering
\caption{Top fragile features before and after VIF pruning for RBF SVM. 
Fragility values represent the variance-to-mean ratio of absolute SHAP values. 
Lower fragility indicates more stable feature attribution. 
Features removed due to high multicollinearity are shown as ``--''.}
\label{tab:fragility_rbfsvm}
\begin{tabular}{lcc}
\toprule
\textbf{Feature} & \textbf{Fragility SVM Full} & \textbf{Fragility SVM VIF} \\
\midrule
proto\_sctp        & $0.555$ & $0.597$ \\
proto\_sun-nd      & $0.492$ & $0.514$ \\
proto\_mobile      & $0.484$ & $0.157$ \\
proto\_br-sat-mon  & $0.416$ & $0.474$ \\
state\_REQ         & $0.407$ & $0.432$ \\
proto\_dcn         & $0.393$ & $0.462$ \\
service\_ssl       & $0.392$ & $0.398$ \\
proto\_a/n         & $0.388$ & $0.447$ \\
proto\_any         & $0.363$ & $0.429$ \\
proto\_emcon       & $0.333$ & $0.390$ \\
proto\_igp         & $0.325$ & $0.390$ \\
proto\_sat-mon     & $0.321$ & $0.386$ \\
proto\_sat-expak   & $0.294$ & $0.356$ \\
proto\_ifmp        & $0.282$ & $0.366$ \\
proto\_srp         & $0.276$ & $0.340$ \\
proto\_wsn         & $0.274$ & $0.364$ \\
proto\_dgp         & $0.267$ & $0.323$ \\
ct\_flw\_http\_mthd & $0.260$ & $0.254$ \\
proto\_larp        & $0.257$ & $0.341$ \\
proto\_pnni        & $0.254$ & $0.304$ \\
\bottomrule
\end{tabular}
\end{table}

\end{document}